\documentclass[sigconf]{acmart}
\usepackage{pgfplots}
\usepackage{float}
\usepackage{thmtools}

\usepackage{caption}
\usepackage{subcaption}
\usepackage{algorithm}
\usepackage[noend]{algpseudocode}
\usepackage{bbm}
\usepackage{algpseudocode}
\usepackage{caption}
\usepackage{subcaption}
\usepackage{booktabs}
\usepackage{amsmath,amsthm}
\newtheorem{definition}{Definition}[section]
\newtheorem{theorem}{Theorem}
\newtheorem{example}{Example}
\newtheorem{remark}{Remark}

\newtheorem{lemma}{Lemma}

\AtBeginDocument{%
  }

\copyrightyear{2023}
\acmYear{2023}
\acmConference[FAccT '23]{2023 ACM Conference on Fairness, Accountability,
and Transparency}{June 12--15, 2023}{Chicago, IL, USA}
\acmBooktitle{2023 ACM Conference on Fairness, Accountability, and
Transparency (FAccT '23), June 12--15, 2023, Chicago, IL, USA}
\acmDOI{10.1145/3593013.3594086}
\acmISBN{979-8-4007-0192-4/23/06}

\begin{document}
%%
%% The "title" command has an optional parameter,
%% allowing the author to define a "short title" to be used in page headers.
\title{Can Querying for Bias Leak Protected Attributes? \\ 
Achieving Privacy With Smooth Sensitivity}

\author{Faisal Hamman}
\affiliation{%
  \institution{University of Maryland}
  \streetaddress{}
  \city{}
  \country{}}
\email{fhamman@umd.edu}

\author{Jiahao Chen}
\affiliation{%
  \institution{Responsible AI LLC }
  \city{}
  \country{}
}
\email{jiahao@responsibleai.tech}

\author{Sanghamitra Dutta}
\affiliation{%
 \institution{University of Maryland}
 \streetaddress{}
 \city{}
 \state{}
 \country{}}
 \email{sanghamd@umd.edu }

% \author{Huifen Chan}
% \affiliation{%
%   \institution{Tsinghua University}
%   \streetaddress{30 Shuangqing Rd}
%   \city{Haidian Qu}
%   \state{Beijing Shi}
%   \country{China}}

% \author{Charles Palmer}
% \affiliation{%
%   \institution{Palmer Research Laboratories}
%   \streetaddress{8600 Datapoint Drive}
%   \city{San Antonio}
%   \state{Texas}
%   \country{USA}
%   \postcode{78229}}
% \email{cpalmer@prl.com}

% \author{John Smith}
% \affiliation{%
%   \institution{The Th{\o}rv{\"a}ld Group}
%   \streetaddress{1 Th{\o}rv{\"a}ld Circle}
%   \city{Hekla}
%   \country{Iceland}}
% \email{jsmith@affiliation.org}

% \author{Julius P. Kumquat}
% \affiliation{%
%   \institution{The Kumquat Consortium}
%   \city{New York}
%   \country{USA}}
% \email{jpkumquat@consortium.net}

%%
%% By default, the full list of authors will be used in the page
%% headers. Often, this list is too long, and will overlap
%% other information printed in the page headers. This command allows
%% the author to define a more concise list
%% of authors' names for this purpose.
\renewcommand{\shortauthors}{Faisal Hamman, Jiahao Chen, and Sanghamitra Dutta}

%%
%% The abstract is a short summary of the work to be presented in the
%% article.
\begin{abstract}
   Existing regulations often prohibit model developers from accessing protected attributes (gender, race, etc.) during training. This leads to scenarios where fairness assessments might need to be done on populations without knowing their memberships in protected groups. In such scenarios, institutions often adopt a separation between the model developers (who train their models with no access to the protected attributes) and a compliance team (who may have access to the entire dataset solely for auditing purposes). However, the model developers might be allowed to test their models for disparity by querying the compliance team for group fairness metrics. In this paper, we first demonstrate that simply querying for fairness metrics, such as, statistical parity and equalized odds can leak the protected attributes of individuals to the model developers. We demonstrate that there always exist strategies by which the model developers can identify the protected attribute of a targeted individual in the test dataset from just a single query. Furthermore, we show that one can reconstruct the protected attributes of \emph{all} the individuals from $\mathcal{O}(N_k \log( n /N_k))$ queries when $N_k\ll n$ using techniques from compressed sensing ($n$ is the size of the test dataset and $N_k$ is the size of smallest group therein). Our results pose an interesting debate in algorithmic fairness: Should querying for fairness metrics be viewed as a neutral-valued solution to ensure compliance with regulations? Or, does it constitute a violation of regulations and privacy if the number of queries answered is enough for the model developers to identify the protected attributes of specific individuals? To address this supposed violation of regulations and privacy, we also propose Attribute-Conceal, a novel technique that achieves differential privacy by calibrating noise to the smooth sensitivity of our bias query function, outperforming naive techniques such as the Laplace mechanism. We also include experimental results on the Adult dataset and synthetic dataset (broad range of parameters).
\end{abstract}
%%
%% The code below is generated by the tool at http://dl.acm.org/ccs.cfm.
%% Please copy and paste the code instead of the example below.
%%
%% Keywords. The author(s) should pick words that accurately describe
%% the work being presented. Separate the keywords with commas.
% \received{20 February 2007}
% \received[revised]{12 March 2009}
% \received[accepted]{5 June 2009}
\maketitle
\section{Introduction}\label{intro}

The ethical goal of algorithmic fairness~\cite{Barocas2016BigDD, megansmith} is closely tied to the legal frameworks of both anti-discrimination and privacy. For instance, Title VII of the Civil Rights Act of 1964 \cite{Barocas2016BigDD} introduces two different notions of unfairness, namely, disparate impact~\cite{10.2307/1072940}, and disparate treatment~\cite{zimmer1995emerging}, which are often at odds with each other~\cite{Barocas2016BigDD}.  %Disparate impact occurs when policies, procedures, and rules have adverse effect on the members of specific groups of people disproportionately based on their protected attributes, such as race, gender, nationality, religion, etc. On the other hand, disparate treatment occurs when there are unequal decisions towards individuals explicitly because of a protected attribute. %Disparate treatment is viewed as an intentional decision to treat individuals differently based on their gender, race, or other protected attributes.
It is widely believed that a machine learning model can avoid violating disparate treatment (and privacy concerns) if the model does not \emph{explicitly} use the protected attributes \cite{Barocas2016BigDD}. However, it has been demonstrated that even if no protected attributes are explicitly used during training, a model might still be held liable for disparate impact, due to proxies of the protected attributes among other attributes in the dataset \cite{10.2307/23015965}. In fact, existing literature on algorithmic fairness demonstrates that \emph{leveraging protected attributes during training can essentially prevent disparate impact}, e.g., by minimizing a fairness metric as a regularizer with the loss function during training \cite{10.1145/2939672.2945386,kamishima2011fairness}. Thus, mitigating disparate impact often seems to be at odds with disparate treatment \cite{lipton2019does} (and privacy), depending on whether the protected attribute is being explicitly used or not. 

One potential resolution (still debated~\cite{pmlr-v80-kilbertus18a}) is to use the protected attributes only during training to mitigate disparate treatment but not after deployment. Nonetheless, \emph{the use of protected attributes during model training remains to be a source of active debate} and contention for various applications \cite{2019,lipton2019does}. On one hand, using protected attributes during training enables one to actively audit and account for biases, as well as understand how specific groups of people are affected. On the other hand, these protected attributes can also be used maliciously, e.g., to exacerbate discrimination~\cite{roberts2014protecting}. An interesting example arises where the protected attributes can even be used to ``mask'' discrimination \cite{dutta2021fairness,datta2017use,Gerrymandering,dutta2020information}, e.g., an expensive housing Ad is shown to only high-income White individuals and low-income Black individuals but not to low-income White individuals and high-income Black individuals (assuming an equal proportion of all these four sub-groups) \cite{dutta2021fairness}. The decision is clearly discriminatory against high-income Black individuals for whom the Ad is relevant and yet they do not get to see it. This discrimination is masked since the decision might still satisfy statistical independence between the two races.
%Ad is an XOR of the race and income of an individual. .... <elaborate>

%\cite{mehrabi2019survey, lipton2019does}

In several applications, e.g., in finance, anti-discrimination and privacy regulations adopt a stance that completely prohibits the use of protected attributes during training. In the finance domain, institutions cannot  ask about an individual’s race for credit decisioning, while at the same time having to prove that their decisions are non-discriminatory \cite{2019J}. The Apple Card credit card was recently accused of discriminatory credit decisioning  since women received lower credit limits than equally qualified men, despite not using the gender explicitly during training \cite{NeilApple}.

Fairness assessment of these models is extremely challenging when the protected attributes are unavailable\footnote{While  \cite{gupta2018proxy} suggests that proxies may be used to approximate missing protected attributes for bias assessment, others~\cite{kallus2020assessing,2019J} argue that  proxies may be problematic since they often underestimate bias.}.
%It enables for an optimal trade-off between accuracy and some fairness constraints when used as regularizers during training \cite{mehrabi2019survey, lipton2019does}. It's also easier to understand how specific individuals are affected when these protected attributes are used for audits. Other pre-processing and post-processing using the protected attributes help mitigate bias \cite{mehrabi2019survey}. 
%While these protected attributes can be used to correct biases, they can also be used to exacerbate discrimination in other situations. In the United States, for example, Asian students were discriminated against in higher education under the pretext of affirmative action \cite{Anemona}. Furthermore, there are cases where protected attribute knowledge can be used to mask disparities \cite{dutta2021fairness}. To demonstrate that the algorithmic decisions comply with these regulations, one needs to do a fairness assessment of these models on populations without knowing their memberships in protected classes. 
To address this, institutions often adopt a separation between the model developers and the compliance team \cite{2019J}. The compliance team is responsible for ensuring methods do not violate anti-discrimination and privacy laws. As a result, the compliance team has access to the entire dataset, including the attributes protected by law (i.e., race, gender, etc.)~\cite{2019J,chen2018fair}. Only a subset of the data fields is visible to the model developers who train these models. The compliance team determines which attributes the model developers are allowed to see and use to train their models. Clearly, the model developers would not have access to the protected attributes. For fairness assessment, the model developers may however query the compliance team for certain group fairness metrics, e.g., statistical parity, equalized odds, etc. The model developers can then choose which model to deploy or discard based on the query responses.

In this paper, we first demonstrate that simply querying for fairness metrics (bias) can also leak the protected attributes of targeted individuals. Furthermore, we demonstrate that there exist strategies by which the model developers can always identify the protected attributes of \emph{all} the individuals in the test dataset. We collectively refer to these strategies as \emph{Attribute-Reveal}. Our finding poses an interesting debate in the policy aspects of fairness and privacy: \emph{Should querying for fairness metrics be viewed as a neutral-valued solution to ensure compliance with regulations? Or, does it constitute a violation of regulations and privacy, particularly if the number of queries answered is enough for the model developers to identify the protected attributes of specific individuals?} To address this supposed violation of regulations and privacy, we also propose \emph{Attribute-Conceal}, a novel differentially-private technique to answer queries without leaking the protected attributes. 
 
To summarize, our main contributions are as follows:\\
 \textbf{1. Demonstrate that querying for bias can leak protected attributes:} We first demonstrate that querying for fairness metrics, e.g., statistical parity, or equalized odds, can indeed leak the protected attributes of individuals. In Theorem \ref{Recovering Sensitive Attributes $A$ using Linear Equations}, we provide the general criterion for reconstructing the protected attributes of all the individuals in the test dataset by querying for the statistical parity of several models (reduces to a linear system of equations).\\
\textbf{2. Leverage compressed sensing to reconstruct protected attributes with fewer queries:} Building on our initial result (Theorem \ref{Recovering Sensitive Attributes $A$ using Linear Equations}), we then demonstrate how protected attributes of individuals can be leaked using a much smaller number of statistical parity queries, provided the size of one group is much smaller than the other (see Theorem \ref{Recovering Sensitive Attributes using Compressed Sensing}). Our findings also extend to the absolute value of statistical parity (see Section \ref{abssec}), as well as, other fairness metrics, e.g., equalized odds or their absolute values. We collectively refer to these proposed reconstruction strategies as \emph{Attribute-Reveal}.\\
\textbf{3. Propose \emph{Attribute-Conceal}, a novel technique that achieves differential privacy by calibrating noise to the smooth sensitivity of our bias query function:} To avoid leaking protected attributes, we propose \emph{Attribute-Conceal}, a technique that answers fairness queries in an $\epsilon-$differentially-private manner (see Section \ref{SmoSen}). Since calibrating noise to global sensitivity (e.g., using the Laplace mechanism) can often hurt the utility of the answered query (because the noise becomes too high), we employ the \textbf{smooth sensitivity} framework~\cite{10.1145/1250790.1250803}, which adds dataset-specific additive noise to achieve differential privacy (see Theorem \ref{caucy} in Section \ref{SmoSen}).\\
\textbf{4. Experiments:} To complement our theoretical results, we also provide experimental results on the Adult dataset~\cite{Dua:2019a} as well as perform simulations on synthetic data for a broad range of parameters. We demonstrate how Attribute-Reveal reconstructs the protected attributes, and how Attribute-Conceal prevents the reconstruction. We also compare Attribute-Conceal to other naive differential privacy techniques such as the Laplace mechanism. \\

\noindent \textbf{Related Works:}
Algorithmic fairness is an active area of research~\cite{varshney2022trustworthy,whittaker2019disability,berk2017fairness,mehrabi2019survey,10.2307/24833783,10.1145/2783258.2783311,RePEc:inm:ormnsc:v:65:y:2019:i:7:p:2966-2981,kamishima2011fairness,NIPS2016_9d268236,2017LL,dwork2011fairness,20192ee,lahoti2019ifair,datta2017use,dutta2021fairness,dutta2020information,varshney2019trustworthy,dutta2020there,coston2019fair,gupta2018proxy,2019J,veldanda2023hyper,alghamdi2022beyond,Gerrymandering} that is receiving increasing attention. Many of these techniques assume that the protected attributes are available during training, which is not always allowed in practice. In several applications, such as credit or loan decisioning \cite{chen2018fair,2019J}, the use of protected attributes during training is restricted by law. Our work is closely related to a body of work that addresses \emph{fairness without access to protected attributes} \cite{gupta2018proxy,2019J,veldanda2023hyper,coston2019fair,sohoni2020no}, often using proxies to estimate the protected attributes. Our work lies in an area that relies on trusted third parties who have access to protected data necessary for improving fairness. For instance, \cite{doi:10.1177/2053951717743530,pmlr-v80-kilbertus18a,hu2019distributed} assume that the model has access to the protected attributes in an encrypted form via secure multi-party computation. Later, \cite{pmlr-v97-jagielski19a} noted that secure multi-party computation technique does not protect protected attributes from leaking, employing differential privacy \cite{10.1007/11681878_14} to learn fair models. \cite{2020LLLL} uses a fully homomorphic encryption scheme, allowing model developers to train models and test them for bias without revealing the protected attributes. More recent works~\cite{10.1145/3314183.3323847, 10.1145/3308560.3317584,alabi2021cost,bagdasaryan2019differential} provide schemes that allow the release of protected attributes privately for learning non-discriminatory predictors. \cite{juarez2022you}  proposes a differentially
private mechanism to measure differences in performance across groups while protecting the privacy of group membership in a federated setting. 

Our work instead addresses a novel problem statement: Can querying for fairness metrics leak protected attributes, and if so, how can we leverage smooth sensitivity to prevent this leakage. Our problem setup also differs from existing works in attribute inference attacks~\cite{DBLP:journals/corr/abs-2103-07853,10.1145/2810103.2813677,DBLP:journals/corr/abs-2005-08679,10.1145/3436755} where the focus is on \emph{learning} protected information in training data from model outputs  using supervised learning (we do not use group membership labels from past data). An interesting related work \cite{yan2022active} studies query-based auditing algorithms to estimate the statistical parity of ML models. Another related work is \cite{shamsabadiwashing}, which focuses on the issue of fair washing, where manipulation techniques are utilized to mask unfairness when presenting the model's explanations to an auditor.

\section{Problem Setup}\label{setting}
%\subsection{Preliminaries}
\subsection{Preliminaries}
Let $\mathcal{S} = (X, Y, A)$ represent a test dataset consisting of $n$ samples, where $A = (a_1, a_2,..., a_j,..., a_n)$ denotes the protected/sensitive attributes (binary), $X = (x_1, x_2,..., x_j,..., x_n)$ denotes the model inputs with each $x_j \in \mathbb{R}^d$, and $Y = (y_1, y_2,..., y_j,..., y_n)$ being the corresponding true labels used for supervised learning.
% Let $\mathcal{S}=(X,Y,A)$ be a test dataset of size $n$ with $A=(a_1,a_2,\ldots,a_j,\ldots,a_n)$ being the protected/sensitive attributes (binary), $X=(x_1,x_2,\ldots,x_j,\ldots,x_n)$ being the model inputs such that each $x_j \in \mathbb{R}^d$, and $Y=(y_1,y_2,\ldots,y_j,\ldots,y_n)$ being the corresponding true labels for supervised learning.
We let $a_j=1$ denote the advantaged group and $a_j=0$ denote the disadvantaged group. %The true labels $Y \in \{0,1\}^n$, where $y_j=1$ represents a favorable outcome, such as loan approval, and $y_j=0$ represents an unfavorable outcome. We assume each data point in the dataset is characterized by a single protected attribute and true label, i.e, $\mathcal{S}=\{(x_j,y_j,a_j)\}_{j=1}^n$.
We consider two types of classifiers: binary and logistic classifiers. The binary classifiers are represented by the function  $h(\cdot): \mathbb{R}^d \rightarrow \{0,1\}.$ %, $ $i \in [m]$, where $[m]$ denotes the set $\{1, 2, . . . ,m\}$ for $m \in \mathbb{N}$. %Given an input $x$, the $i$-th classifier $h_i(\cdot)$ outputs $1$ if accepted and $0$ if rejected. 
For a logistic classifier  $h(\cdot): \mathbb{R}^d \rightarrow [0,1],$ the output corresponds to the probability of input $x$ being accepted. When we have several classifiers, we let $h_i(x_j)$ represent the $i$-th classifier's output to the input $x_j$ (input feature vector of the $j$-th individual in the dataset) for $i \in [m]$, where $[m]=\{1, 2, . . . ,m\}$ for a positive integer $m$. Let $h_i(X)$ represent the $i$-th classifier's outputs to all individuals in the dataset, i.e, $h_i(X)=\big(h_i(x_1),...,h_i(x_j),...,h_i(x_n)\big)$. We let $N_0$ and $N_1$ be the size of the disadvantaged and advantaged group, i.e., $N_0=\sum_{j \in [n]} \mathbbm{1} \{a_j=0\}$ and $N_1=\sum_{j \in [n]} \mathbbm{1} \{a_j=1\}$. Note that $N_1+N_0=n$, the size of the test dataset. \\

\noindent \textbf{Review of Relevant Group Fairness Metrics:}
\begin{definition}[Statistical Parity Gap ($SP_i$)]
  Statistical parity gap ($SP_i$) is defined as the difference in expected outcome between the advantaged and disadvantaged groups, i.e.,
% \begin{align}\label{defn:Statistical parity gap}
% SP_i & = \mathbb{E}(h_i(\cdot)|A=1)-\mathbb{E}(h_i(\cdot)|A=0) \nonumber \\ 
% & = \frac{1}{N_1} \sum_{\{j|a_j=1\}} h_i(x_j)-\frac{1}{N_0}\sum_{\{j|a_j=0\}} h_i(x_j).
% \end{align}
\begin{equation}
\label{defn:Statistical parity gap}
SP_i = \frac{1}{N_1} \sum_{\{j|a_j=1\}} h_i(x_j)-\frac{1}{N_0}\sum_{\{j|a_j=0\}} h_i(x_j).
\end{equation}
\end{definition}
 \begin{definition}[Equal Opportunity Gap ($EO_i$)]
 Equal opportunity gap ($EO_i$) is defined as:
    % \begin{align}\label{defn:Equality of opportunity gap}
    %  EO_i & = \mathbb{E}(h_i(\cdot)|A=1,Y=1)-\mathbb{E}(h_i(\cdot)|A=0,Y=1) \nonumber \\
    %  & = \frac{1}{N_{11}} \sum_{\{j|y_j=1,a_j=1\}} h_i(x_j)-\frac{1}{ N_{10}}\sum_{\{j|y_j=1,a_j=0\}} h_i(x_j)
    % \end{align}
    \begin{equation}\label{defn:Equality of opportunity gap}
     EO_i = \frac{1}{N_{11}} \sum_{\{j|y_j=1,a_j=1\}} h_i(x_j)-\frac{1}{ N_{10}}\sum_{\{j|y_j=1,a_j=0\}} h_i(x_j),
    \end{equation}
where $ N_{11}= \sum_{j \in [n]} \mathbbm{1} \{ y_j{=}1,a_j{=}1\}$, $N_{10}= \sum_{j \in [n]} \mathbbm{1} \{ y_j{=}1,a_j{=}0\}$.
  \end{definition}
We also denote the absolute values of statistical parity and equal opportunity gap as $|SP_i|$ and $|EO_i|$ respectively. 

\begin{remark}
We note that although we only define statistical parity and equal opportunity, our techniques can be extended to several other group fairness measures, such as equalized odds, predictive parity, etc. as well as their absolute values. We further discuss this in Remark \ref{rem:extend}.
\end{remark}

\subsection{Problem Statement}
  Institutions often adopt a separation between the model developers and the compliance team to ensure anti-discrimination and privacy laws are met. The model developers do not have access to protected attributes and therefore cannot use them for training. The compliance team, however, has access to the entire dataset, but only for auditing purposes. In our setting, the model developers train $m$ different classifiers $h_i(\cdot)$ with $i\in [m]$ on the training dataset. For the fairness assessment of these models before deployment, the model developers are allowed to test for algorithmic bias by querying the compliance team for certain fairness metrics on the test dataset $\mathcal{S}=(X,Y,A)$. Note that the classifier $h_i(\cdot)$ is only a function of the input $x_j$ and not the protected attributes $a_j$. The fairness metrics that the model developers can query for includes the statistical parity gap, and equal opportunity gap as well as their absolute values (see system model in Figure \ref{fig:sysmodel}). % of these measures, i.e., Definitions~\ref{defn:Absolute statistical parity gap} and \ref{defn:Absolute Equality of opportunity gap}. 
 The main question that we ask in this work is: Is this technique of querying for fairness metrics effective in keeping the protected attributes hidden from the model developers? Or in general, \textit{does querying for fairness leak protected attributes?} And if so, how can one answer queries without leaking protected attributes?
%\section{How Can Querying for Fairness Metrics Leak Protected Attributes?}

\begin{remark}
    The approach outlined in this paper can be extended to a scenario with an institution (that trains a model without access to protected attributes) and an external fairness auditing team, e.g., in \cite{yan2022active}. The auditing team has access to the entire data for the purpose of evaluating the bias of the models and is responsible for informing the model developers on whether their deployed model passes fairness tests based on some fairness metric. In this setting, our concern is whether auditing for fairness compromise and leak the protected attributes to the institution.
\end{remark}
\begin{figure}[t]
\includegraphics[scale=0.28]{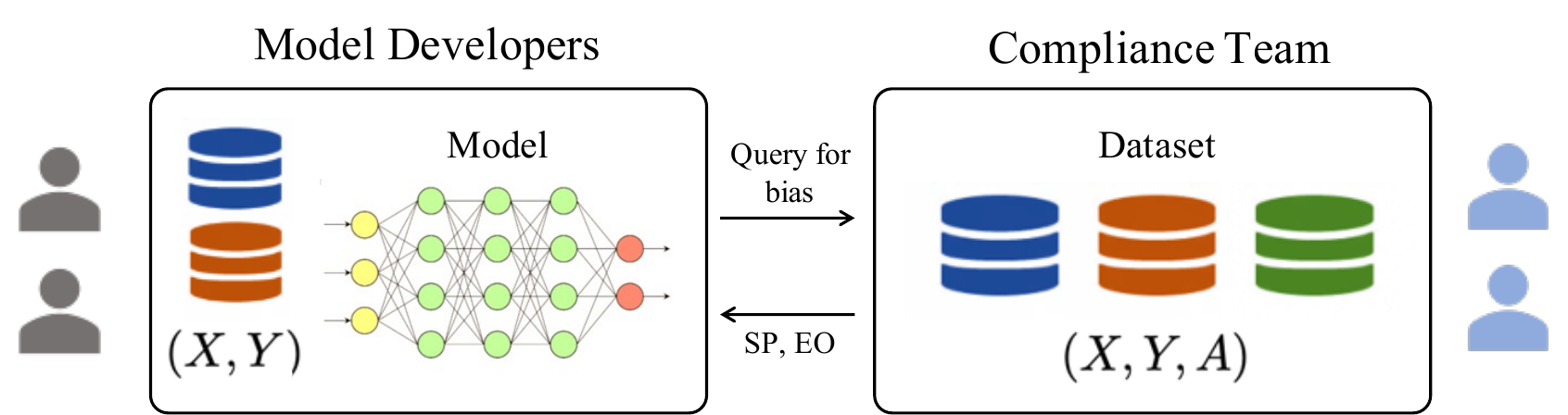}
\caption{Illustrates an institutional separation  between model developers and compliance team for ensuring fair and privacy-compliant machine learning models. Model developers train classifiers on the training dataset, but without access to protected attributes. Compliance team has access to the entire dataset for auditing purposes and is queried by model developers for fairness metrics. }
\label{fig:sysmodel}
\end{figure}
\section{Attribute-Reveal: Querying For Bias Leaks Protected Attributes}
\subsection{Demonstrating Leakage From Querying}
\label{SubSA}
Here, we show that simply querying for fairness metrics such as statistical parity can reveal the protected attribute of any targeted individual to the model developers. In fact, there exist strategies (that we collectively refer to as Attribute-Reveal) that can reveal the protected attributes of all the individuals in the test dataset. We begin with a simple toy example.
\begin{example}[Single Query]\label{1acceptexample}
The model developers train only one binary classifier $h_1(\cdot)$ and query the compliance team for statistical parity gap.  Suppose, the model developers want to find the protected attribute of the first individual. They can choose a classifier that accepts only the first individual, i.e., $h_1(x_1)=1$ and $h_1(x_j)=0$ for $j=2,3,\ldots,n$. Observe that the statistical parity gap of this model will reveal the protected attribute of the first individual as follows: $SP_1= \frac{ 1}{N_1}$ if $a_1=1$, and $SP_1=-\frac{1}{N_0}$ if $a_1=0$.
% \begin{equation}\label{onequery}
% SP_1= 
% \begin{cases}
%       \frac{ 1}{N_1},& \text{\textit{if} } a_1=1\\
%     -\frac{1}{N_0}, & \text{\textit{if} } a_1=0
% \end{cases}
% \end{equation}
Thus, a positive statistical parity gap $SP_1$ would give away that the individual belongs to the advantaged group, whereas a negative gap indicates that the individual belongs to the disadvantaged group. The query also reveals the sizes of these groups $N_1$ and $N_0$. \end{example}
We note that such a model might seem contrived; it might also have low accuracy. Thus, in Example~\ref{multipleacceptexample}, we demonstrate a more realistic scenario that can occur more commonly in practice: \emph{two models of comparable accuracy} can be used to reveal the protected attribute of a targeted individual.
\begin{example}[Double Query With Realistic Models]\label{multipleacceptexample}
 Consider two models $h_1(\cdot)$ and $h_2(\cdot)$. The model $h_1(\cdot)$ is trained by model developers (to maximize accuracy) and accepts several individuals in the dataset. The model developers also use a second classifier, $h_2(\cdot)$, that provides the same prediction as $h_1(\cdot)$ except for one targeted individual, i.e., $h_2(x_1)=1-h_1(x_1)$ and $h_2(x_j)=h_1(x_j)$ for $j=2,3,\ldots,n$. Notice that $h_1(\cdot)$ and $h_2(\cdot)$ differ \emph{only} in the first prediction. Their accuracies are almost similar. If one queries for statistical parity of these two models, they can identify the protected attribute of the first individual as follows: $SP_2-SP_1= \frac{ 1}{N_1}$ if $a_1=1$ and $SP_2-SP_1= -\frac{ 1}{N_0}$ if $a_1=0$.
% $$SP_2-SP_1= 
% \begin{cases}
%       \frac{ 1}{N_1},& \text{\textit{if} } a_1=1\\
%     -\frac{1}{N_0}, & \text{\textit{if} } a_1=0
% \end{cases}$$
\end{example}
A similar approach can be adopted to reveal the protected attributes of all the individuals in the test dataset. 
%If multiple models were considered and the predictions of some of them differed only by one individual, then the protected attributes of those individuals where the predictions differed would leak. Now, we generalize this idea.
%\begin{example}
%\rd{SD: This example can be made more realistic - we can describe exactly what you do in experiments - have a model with good accuracy, and then do this flipping}
% Assume we have $m=n$, i.e., we have as many models as the number of individuals in the dataset. Suppose each model $h_i(\cdot)$ accepts only the $j$-th individual for $i=j$ and rejects everyone else, i.e., we have
% $$h_i(x_j)=
% \begin{cases}
%       1,& \text{\textit{for} } i=j\\
%     0, & \text{\textit{for} } i\neq j
% \end{cases}.
% $$
% By using the idea in Example~\ref{1acceptexample}, the protected attributes of all the individuals in the dataset will be identified from these queries of the models. Similarly, if we have $m=n$ models where the predictions of each model differed by one unique element, the protected attributes will be revealed by continuously taking the difference between the queries of two different models, as seen in Example~\ref{multipleacceptexample}. In a practical setting, one model $h_1(\cdot)$ could be  trained with good accuracy, and subsequent models $h_i(x_j)$ could be chosen so the prediction of $h_1(\cdot)$ is flipped when $i=j$. This will allow recovery of protected attributes while keeping the accuracy of the models approximately the same.
% \end{example}
Our next result provides the general criterion for reconstructing the protected attribute of all the individuals in the test dataset using the statistical parity queries (see Appendix~\ref{1qw} for proof).

\begin{restatable}[Reveal From Linear System of Equations] {theorem}{thrmone}
\label{Recovering Sensitive Attributes $A$ using Linear Equations}
Let $\overline{SP}=\begin{bmatrix} 
SP_1 & SP_2 & \ldots & SP_m\\
\end{bmatrix}^T $ be a vector of statistical parity gap queries for $m$ models, $h_i(x_j)$ denote the $i$-th model's prediction for the $j$-th individual, and $\bf H$ be an $m\times n$ matrix where each row represents the binary or logistic predictions of the $i$-th model. If $\textit{rank}(\textbf{H} )= n$, then the protected attributes $A = (a_1, a_2,..., a_j,..., a_n)$ of the entire dataset can be identified by solving a linear system of equations:  
% $\overline{SP}= \textbf{ \textup H}_{m \times n}\textrm{\textbf{\textup v}}$, where,
% $\textrm{\textbf{\textup v}} $ is the unknown vector with elements taking values: $v_j=\frac{ 1}{N_1}$ if $a_j=1$ and $-\frac{1}{N_0}$ if $a_j=0$.
% $$\overline{SP}= \textbf{ \textup H}_{m \times n}\textrm{\textbf{\textup v}} \label{lse}$$
\begin{equation}
\begin{pmatrix} SP_{1} \\ SP_{2} \\ . \\ . \\ . \\ SP_{m} \end{pmatrix}
 =
  \begin{pmatrix}
  h_1(x_1) & h_1(x_2) & ... &...& h_1(x_n)\\
  h_2(x_1) & h_2(x_2) & ... &...& h_2(x_n)
  \\.& . & &&.\\.&  &. &&.\\.&  & &.&.\\
  h_m(x_1) & h_m(x_2) & ... &...&h_m(x_n)
  \end{pmatrix}
  \begin{pmatrix} v_{1} \\ v_{2} \\ . \\ . \\ . \\ v_{n} \end{pmatrix} .
\end{equation}
This can also be expressed as, 
\begin{equation}\label{lse}
    \overline{SP}= \textbf{ \textup H}_{m \times n}\textrm{\textbf{\textup v}},
\end{equation}
where
$\textrm{\textbf{\textup v}} $ is the unknown vector with elements taking values: 
$$v_j= 
\begin{cases}
       \;\; \frac{ 1}{N_1},& \text{\textit{if } } a_j=1\\
    -\frac{1}{N_0}, & \text{\textit{if } } a_j=0.
\end{cases}
$$
\end{restatable}
% \ref{Recovering protected Attributes $A$ using Linear Equations} 

\begin{remark}
Strictly speaking, one needs $m=n-1$ queries since the last individual can be identified using group sizes $N_1$ and $N_0$. However, one may encounter numerical errors when solving for $N_1$ from $1/N_1$ or, $N_0$ from $1/N_0$. This can happen when $N_1,N_0>>1$ and $N_0 \approx N_1$.
\end{remark}
Algorithm~\ref{algo:reveal_full_rank} provides a more realistic strategy by which model developers can choose practical models with comparable accuracy to reveal the protected attributes of all the individuals. Essentially, one base model $h_0(\cdot)$ could be trained, and several similar models $h_i(x_j)$ could be chosen so the prediction of $h_i(x_j)$ is flipped only when $i=j$. The accuracy of these models would remain comparable to the base model $h_0(\cdot)$ since they differ in only one prediction.

% \begin{minipage}{0.49\textwidth}
\begin{algorithm}[t]
\centering
\caption{Attribute-Reveal ($m=n$)}\label{algo:reveal_full_rank}
\begin{algorithmic}
\State Train base model $h_0(\cdot)$ (reasonable accuracy)
\State Choose $m{=}n$ models $h_1(\cdot)$, \ldots, $h_n(\cdot)$ as:
\For{$i=1,2,\ldots,m$} \For{$j=1,2,\ldots,n$}

 $h_i(x_j)= \begin{cases}
   h_0(x_j),& \text{\textit{if } } i\neq j\\
   1-h_0(x_j), & \text{\textit{if } } i= j.
\end{cases}$
\EndFor
\EndFor
\State Query for Statistical Parity Gap for each model and store in $\overline{SP}$
\State Create matrix:
${\bf H}_{m \times n}=[h_1(X)^T,h_2(X)^T,\ldots,h_m(X)^T]^T$
\State Solve linear system of equations: 

$\overline{SP}={\bf Hv}$
\Comment{This algorithm is based on Theorem \ref{Recovering Sensitive Attributes $A$ using Linear Equations}}
\For{$j=1,2,\ldots,n$} \State Detect $\hat{a}_j=\begin{cases}
1, &\text{if } v_j >0\\
0, &\text{otherwise}.
\end{cases}$
\EndFor
\end{algorithmic}
\end{algorithm}
% \end{minipage}
% \hfill
% \begin{minipage}{0.49\textwidth}
% \begin{algorithm}[t]
% \centering
% \caption{Attribute-Reveal ($m\ll n)$}\label{algo:reveal_cs}
% \begin{algorithmic}
% \State Train base model $h_0(\cdot)$ (reasonable accuracy)
% \State Choose $m{=}k$ models $h_1(\cdot)$, \ldots, $h_k(\cdot)$ as: 
% \For{$i=1,2,\ldots,m$}
% \For{$j=1,2,\ldots,n$}

%  Sample $N\sim$ Uniform$(-0.1, 0.1)$ 

%  $h_i(x_j)=   h_0(x_j) + N \ \ $(Clip in $[0,1]$)
% \EndFor
% \EndFor
% \State Query for Statistical parity gap for each model and store in $\overline{SP}$
% \State Create sensing matrix ${\bf H}_{m \times n}$

% $=[h_1(X)^T,h_2(X)^T,\ldots,h_m(X)^T]^T$ 
% \State Compute $\eta = \overline{SP}- \textbf{H} \bar r$
% \State %Find $\bar s$ using compressed sensing:  
% Solve: $\min_{\bar s} {\|\bar s\|_1} \textrm{\;\;s.t.\;\;}  \eta = \textbf{H} \bar s $
% \For{$j=1,2,\ldots,n$} \State Detect $\hat{a}_j{=}\begin{cases}
% 0,  \text{ if } {\bar{s}_j > 0.5(\frac{1}{N_1} +\frac{1}{N_0})}\\
% 1, \text{ otherwise}.
% \end{cases}$
% \Comment{This algorithm is based on Theorem \ref{subsec:cs}}
% \EndFor
% \end{algorithmic}
% \end{algorithm}
% \end{minipage}

We note that if the size of the test dataset is large, it may not be desirable to have as many models as the size of the dataset. This motivates our next question: \textit{Is it possible to obtain the protected attributes of individuals in the dataset with fewer models and queries? } %Section~\ref{subsec:cs} introduces an improved technique of reconstructing the protected attributes of all the individuals that leverages compressed sensing to significantly reduce the number of queries $m$. (e.g., the adult dataset with $n\approx 50000$ \cite{Dua:2019})
\subsection{Leaking Protected Attributes with Fewer Queries using Compressed Sensing (CS)}
\label{subsec:cs}
In this section, we demonstrate that compressed sensing (CS) techniques can be used to obtain the protected attributes of individuals using a significantly smaller number of queries ($m$). First, we  provide a brief background on compressed sensing in Section \ref{ComBack}. Readers already familiar with this topic may skip this subsection.
% The following theorem formalizes our result. For a background on compressed sensing, we refer the reader to Section~\ref{ComBack}.
\subsubsection{Brief Background on Compressed Sensing}\label{ComBack}
The goal~\cite{1614066,lotfi2020compressed} is to recover a vector $x \in \mathbb{R}^n$ from a set of linear measurements $\eta = \Phi x$, where $\eta $  is an $m\times 1$ measurement vector,  $\Phi \in \mathbb{R}^{m \times n}$  is the sensing matrix. CS relies on the sparsity of $x$. \emph{A vector $x$ is  $k$-sparse if it  has only $k$ non-zero entries (typically $k\ll n$)}. The measurements $m$ are typically much smaller than $n$, making this an under-determined system of equations, having many solutions for $x$.
CS focuses on finding the sparsest solution for $x$. This can be expressed as an optimization problem:
$ \min_{x}  ||x||_0 \ \ 
\textrm{s.t.}  \ \ \eta = \Phi x$. Here, $||x||_0$ is the number of nonzero entries\footnote{Note that  $||x||_0$ is not a norm. $||x||_p$ denotes a standard $l_p$-norm for $p \geq 1$, i.e., $||x||_p=\left(\sum_{i=0}^n |x_i|^p \right)^\frac{1}{p}$ for all $x \in \mathbb{R}^n$} of $x$. By minimizing an $l_1$ norm instead, this problem can be relaxed into a convex optimization problem which can be solved using linear programming or other CS algorithms, e.g., Orthogonal Matching Pursuit ~\cite{DBLP:journals/corr/MairalBP14}.
\begin{equation}
\min_{x}  ||x||_1 \ \
\textrm{s.t.} \ \ \ \ \eta = \Phi x.
\label{eq:CS_l1}
\end{equation}
The $l_1$-norm $||x||_1$ is the absolute sum of all entries of $x$. We refer the reader to an excellent survey \cite{4472240} for more information on CS.

For accurate recovery of $k$-sparse vector $x$ from measurements $\eta $, the sensing matrix $ \Phi $ has to satisfy a necessary and sufficient condition called Restricted Isometry Property (RIP)~\cite{candes2005decoding}.

\begin{definition}[$k$-Restricted Isometry Property]\label{def:RIP}
 A matrix $\Phi \in \mathbb{R}^{m \times n}$ satisfies the Restricted Isometry Property of order $k$ if for all $k$-sparse vector $x \in \mathbb{R}^n$, and for some constant  $\delta_k \in (0,1)$, we have
\begin{equation}\label{RIP}
(1-\delta_k)||x||_2^2\leq ||\Phi x||_2^2\leq (1+\delta_k)||x||_2^2.
\end{equation}
\end{definition}
For a matrix $\Phi$ that satisfies the RIP condition of order $2k$ with $\delta_{2k}<\sqrt{2}-1$ (see \cite{CANDES2008589}), the vector $x$ can be reconstructed from $\eta$ and $\Phi$ by solving \eqref{eq:CS_l1}. Random matrices satisfy RIP of any order $k$ with high probability provided that $m = \mathcal{O}(k \log(n/k))$ \cite{baraniuk2008simple,candes2005stable}. Therefore, provided $x$ is sufficiently sparse, smaller measurements $m$ suffice to ensure a high-quality reconstruction of $x$. It is also known that any CS algorithm will require at least $m = \Omega(k \log(n/k))$ for reconstruction \cite{10.2307/40587229}. 

In general, designing or checking whether a sensing matrix satisfies the RIP condition is computationally difficult. RIP only gives a condition on whether a matrix can be used as a sensing matrix but does not necessarily mention how to design one in practice. There are certain random matrices that are known to satisfy the RIP condition with high probability \cite{lotfi2020compressed,4016283}. The most common is the Gaussian matrix, i.e., $\Phi_{m \times n}$ consists of $mn$ independent samples from a zero-mean Gaussian distribution with a variance of $1/m$ \cite{4016283}. The random binary matrix is another well-studied sensing matrix that is known to satisfy the RIP condition \cite{5512379}. 

\subsubsection{Reveal from Compressed Sensing}

\begin{theorem}[Reveal From Compressed Sensing]\label{Recovering Sensitive Attributes using Compressed Sensing}
Assume that $N_0 \ll N_1$, i.e., the size of the disadvantaged group in the dataset is much smaller than the advantaged group, and\:  $\textbf{H}_{m\times n} $ is a random matrix strongly concentrated around its mean. 
% that satisfies the $N_0$-Restricted Isometry Property in Definition \ref{def:RIP}. 
Then, the protected attribute vector $A = (a_1, a_2,..., a_j,..., a_n)$ of the entire test dataset can be obtained using $m=\mathcal{O}(N_0 \log(n/N_0))$ statistical-parity-gap queries. \end{theorem}

%Old theorem 
% \begin{theorem}[Reveal From Compressed Sensing]\label{Recovering Sensitive Attributes using Compressed Sensing}
% Assume that $N_0 \ll N_1$, i.e., the size of the disadvantaged group in the dataset is much smaller than the advantaged group, and\:  $\bf H $ is a binary matrix generated randomly. Then by using compressed sensing, the protected attributes $A$ of the entire test dataset can be obtained using $m=\mathcal{O}(N_0 \log(n/N_0))$ statistical-parity-gap queries. \end{theorem} 

% \textcolor{red}{(we refer readers to section ... for a background on ...)}
\begin{proof} To prove Theorem \ref{Recovering Sensitive Attributes using Compressed Sensing},  we convert \eqref{lse} into a compressed sensing problem (we refer the reader to Section~\ref{ComBack} for a background on compressed sensing). Recall from the proof of Theorem~\ref{Recovering Sensitive Attributes $A$ using Linear Equations} that the reconstruction of the protected attributes reduces to solving the linear system of equations in \eqref{lse}, i.e., $\overline{SP}= {\textbf{H}_{m \times n}} \textrm{\textbf{v}}.$

Compressed sensing allows the number of queries $m$ to be much less than $n$ by exploiting the sparsity of one group in the dataset. Let $\textrm{\textbf{v}}=\bar r - \bar s $ where  $\bar r = \begin{pmatrix} 1/N_1 & 1/N_1 & \ldots & 1/N_1 \end{pmatrix}^T$ and,
    $$ \bar s_j= 
\begin{cases}
      \;\;\;\;0\;\;\;\;\;\;,& \text{\textit{if} } a_j=1\\
    \frac{1}{N_1}+\frac{1}{N_0}, & \text{\textit{if} } a_j=0
\end{cases}$$
We have,
$
\overline{SP}=  \textbf{H} (\bar r - \bar s )
$, leading to
$
\overline{SP}- \textbf{H} \bar r=  \textbf{H} \bar s 
$. Now, let $\eta = \overline{SP}- \textbf{H} \bar r$. Then, we have,
\begin{equation}\label{cseq}
   \eta =  \textbf{H}_{m \times n} \bar s. 
\end{equation}
By simple manipulations, we converted \eqref{lse} into a standard compressed sensing problem \eqref{cseq} where, $\eta$ is the measurement vector, $\textbf{H}$  is the sensing matrix, and $\bar s$ is a sparse  vector with $N_0$ non-zero entries. From Theorem $1.2$  in \cite{CANDES2008589}, the unknown vector $\bar{s}$ can be recovered if $\textbf{H}$ satisfies RIP (see Definition~\ref{def:RIP}) of order $2N_0$ with constant $\delta_{2N_0} <  \sqrt{2}-1$. Furthermore, Theorem \ref{trn:RIP} in Appendix~\ref{apx:col1} shows that a random matrix that is strongly concentrated around its expectation  
satisfies RIP of order $2N_0$ with high probability provided $N_0 \leq c m / \log (n / N_0)$ (implies $2N_0 \leq c' m / \log (n / (2 N_0))$) for constants $c,c'>0$. 
Therefore  $m = \mathcal{O}(N_0 \log(n/N_0))$ statistical parity gap queries suffice to successfully reconstruct the protected attribute vector $A$ of the entire test dataset.
\end{proof}

 \begin{remark}
For the model developers to use the CS technique, they also might need to know $N_1$ and $N_0$. This can be found by querying using a model that only accepts one individual and first checking the sign of $SP_1$. If $SP_1>0,$ then we know $SP_1= \frac{1}{N_1}$ and $a_1=1$, and hence we can get $N_1$. We can also obtain $ N_0=n-N_1.$ Alternatively, if $SP_1<0,$ we know $SP_1=-\frac{1}{N_0}$ (and $a_1=0$), and we can get $N_0$. We can also obtain $ N_1=n-N_0.$
 \end{remark}
% The CS sparsity requirement is satisfied if $N_1$ or $N_0 \ll n$.
To effectively apply CS, the vector $\bar s$ must be sparse, meaning the size of one group (advantaged or disadvantaged) in the dataset must be significantly smaller than the other group. The sparsity requirement does not have a specific strict threshold. However, the smaller the  minority group, the better CS performs with $m=\mathcal{O}(N_0 \log(n/N_0))$ models. As the size of the minority group increases, more models are needed.

The sensing matrix $\textbf{H}$ should satisfy the Restricted Isometry Property (RIP) for CS to work (see Definition~\ref{def:RIP} in Section~\ref{ComBack}). Random binary matrices are well known to satisfy this property with high probability \cite{5512379}. Therefore, choosing models that predict $\{0,1\}$ randomly would work. However, this might lead to unrealistic models that have low accuracy (since essentially it means that model developers are choosing models with random predictions). Gaussian noise is also proven to satisfy the RIP condition but it may result in a model prediction that lies outside the range of $[0,1]$, making the models unrealistic. Even if we clip the values to lie between $[0,1]$, the Gaussian noise may lead to a large variation in accuracy from the base model.

Thus, in Lemma~\ref{col:uni_rip}, we show that a sensing matrix whose entries are independently sampled from a uniform distribution with variance $1/m$ will satisfy the RIP property needed for CS. In practice, this motivates us to use small bounded noise so that the output values deviate as little as possible (Algorithm~\ref{algo:reveal_cs}).

\begin{restatable}{lemma}{uniformRip}\label{col:uni_rip}
Let $\Phi \in \mathbb{R}^{m \times n}$ be a random sensing matrix whose entries are drawn from an i.i.d Uniform$\left(-\sqrt{3/m},\sqrt{3/m}\right)$ distribution, then the  matrix $\Phi$ satisfies the Restricted Isometry Property (RIP) of order $k\leq c_1 m/\log(n / k)$ with at least probability $1-2 e^{-c_2 m}$, for some constant $c_1, c_2 >0$.
\end{restatable}
See proof in Appendix \ref{apx:col1}.

This motivates Algorithm~\ref{algo:reveal_cs}, a novel and realistic strategy by which model developers can choose practical models with comparable accuracy to reveal the protected attributes of all the individuals. Essentially, one base model $h_0(\cdot)$ could be trained. For the other $m$ models, small noise sampled from a uniform distribution is added to each output of the base model. If an output value goes outside $[0,1]$, clip the value to lie between $[0,1]$.

\begin{algorithm}[t]
\centering
\caption{Attribute-Reveal ($m\ll n)$}\label{algo:reveal_cs}
\begin{algorithmic}
\State Train base model $h_0(\cdot)$ (reasonable accuracy)
\State Choose $m$ models $h_1(\cdot)$, \ldots, $h_m(\cdot)$ as: 
\For{$i=1,2,\ldots,m$}
\For{$j=1,2,\ldots,n$}

Sample $n_{ij}\sim$ Unif$\left(-b, b\right)$ for a constant $b$

 $h_i(x_j)=   h_0(x_j) + n_{ij}\ \ $(Clip in $[0,1]$)
\EndFor
\EndFor
\State Query for Statistical parity gap for each model and store in $\overline{SP}$
\State Create sensing matrix ${\bf H}_{m \times n}=[h_1(X)^T,h_2(X)^T,\ldots,h_m(X)^T]^T$ 
\State Compute $\eta = \overline{SP}- \textbf{H} \bar r$
\State %Find $\bar s$ using compressed sensing:  
Solve: $\min_{\bar s} {\|\bar s\|_1} \textrm{\;\;s.t.\;\;}  \eta = \textbf{H} \bar s $
\For{$j=1,2,\ldots,n$} \State Detect $\hat{a}_j{=}\begin{cases}
0,  \text{ if } {\bar{s}_j > 0.5(\frac{1}{N_1} +\frac{1}{N_0})}\\
1, \text{ otherwise}.
\end{cases}$\\
\Comment{This algorithm is based on Theorem \ref{Recovering Sensitive Attributes using Compressed Sensing}}
\EndFor
\end{algorithmic}
\end{algorithm}

% We note that while binary random matrices provably satisfy the RIP property, for this particular application, they might lead to unrealistic models that have low accuracy (since essentially it means that model developers are choosing models with random predictions). Thus, in Algorithm~\ref{algo:reveal_cs}, we provide a novel and more realistic strategy by which model developers can choose practical models with comparable accuracy to reveal the sensitive attributes of all the individuals. Essentially, one base model $h_0(\cdot)$ could be trained. For the other $m$ models, we add a small noise sampled from Uniform$(-0.1,0.1)$ distribution to each output of the base model. 

\subsection{Extension to Absolute Statistical Parity Gap }\label{abssec}
%Here, we look into testing for bias by answering absolute statistical parity queries defined in \eqref{Absolute statistical parity gap}. Absolute statistical parity $|SP_i|$ maps the bias of a model to $[0,1]$ without preserving the sign of statistical parity gap \eqref{Statistical parity gap}. The sign is useful in distinguishing between the advantaged and disadvantaged groups. 

With absolute statistical parity, it is still possible to partition individuals into two groups but no longer possible to determine which group represents the advantaged or disadvantaged populations with certainty. Being able to partition individuals in the test dataset based on their protected attributes is still a privacy infringement. This is mostly due to the ease with which the advantaged and disadvantaged groups can be identified. If somehow the model team could only tell the protected attribute of one individual in a group, e.g., from the other attributes, the partitioning would allow them to learn the protected attribute of the entire test dataset. The partitioned sizes can also be used to determine which group is which. In many cases, the disadvantaged group is often known to be significantly smaller than the advantaged group.
%For example, a model $h_1(\cdot)$ that accepts only one individual would have $|SP_1|= 1/N_1$ or $1/N_0$. However, because there is no sign, the protected attribute of that one individual is still hidden. Though, it is still possible to partition the dataset into two groups based on whether the value is $1/N_1$ or $1/N_0$ (for $N_1 \neq N_0$). Our next result generalizes this notion irrespective of whether $N_0$ and $N_1$ are equal or not.
\begin{theorem}\label{dde}
Given $m=n$ absolute-statistical-parity-gap queries, there exists a strategy that partitions individuals in the test dataset into two different groups based on their protected attributes.  
\end{theorem}

\begin{proof}
We discuss such a strategy in the proof.
Let us use $\alpha$ and $\beta$ to represent the two partitions of the dataset, i.e.,  $A\in \{\alpha,\beta\}^{n}$. Let $N_\alpha$ and $N_\beta$ denote the size of $\alpha$ and $\beta$ partitions respectively. Note that $N_\alpha+N_\beta=n$.

First, obtain $N_\alpha$ and $N_\beta$. This can be done by querying a model that accepts only one individual. The query will return $|SP_1|=1/N_\alpha$ or $1/N_\beta$, revealing the size of the partitions. Now, consider the two cases.\\
\noindent Case 1: $N_\alpha \neq N_\beta$. If the size of the two groups is not equal, query a model $h_1(X)$ that accepts only the first individual in the dataset $x_1$. Assume that the individual belongs to the $\alpha$ partition, i.e., $a_1=\alpha$. $|SP_1|$ would therefore be $1/N_\alpha$. Then, query a second model, $h_2(X)$, that accepts only the second individual $x_2$. If $|SP_2|=1/N_\alpha$, then $a_2=\alpha$. If $|SP_2| \neq 1/N_\alpha$ then $|SP_2|$ must equal $1/N_\beta$, implying that $a_2=\beta$. Continue this procedure for every individual until everyone is classified into $a_j=\alpha$ or $\beta$. 

\noindent Case 2: $N_\alpha = N_\beta$. If the size of the two groups is the same, it would not be possible to differentiate between $1/N_\alpha$ and $1/N_\beta$. Hence, a slightly different approach is taken. First, query a model $h_1(X)$ that only accepts $x_1$ and assume $a_1=\alpha$, resulting in $|SP_1|=1/N_\alpha$. Next, query a second model $h_2(X)$ that accepts only $x_1$ and $x_2$.  The protected attribute of $x_2$ can be obtained using the query $|SP_2|$, i.e.,
$$|SP_2|= 
\begin{cases}
     \;\;\;\;\;  \frac{2}{N_\alpha},& \text{\textit{if } } a_2=\alpha\\
    \frac{1}{N_\alpha}-\frac{1}{N_\beta}=0 , & \text{\textit{if } } a_2=\beta
\end{cases}
$$
In general, to obtain the group of the $j$-th individual $a_j$, select a model that accepts only $x_1$ and $x_j$. To partition the whole dataset using this technique, the model developers would need at most $m=n$ models and queries.
\end{proof}

\begin{remark}\label{rem:extend}
Our results extend to other fairness metrics, such as equalized odds, equal opportunity, and predictive rate parity \footnote{A classifier satisfies equalized odds if the individuals in the advantaged and disadvantaged groups have equal expected outcomes given their true labels. A classifier satisfies predictive rate parity if both groups have an equal probability of a subject with positive predictive value truly belonging to the positive class
\cite{8452913}.}. However, when querying for measures like equal opportunity, the model developers can only identify the protected attributes of individuals with true label $Y=1$. Since equal opportunity conditions on $Y=1$, one does not get any information about individuals with $Y=0$. %As a result, those individuals have their protected attributes hidden.
\end{remark}

\section{Differentially-Private Approaches to Bias Assessments}
In this section, we discuss approaches to prevent the problem of leaking protected attributes. The main goal is to answer fairness queries as accurately as possible but without leaking the protected attributes of any individual in the test dataset. This motivates us to leverage differential privacy \cite{10.1007/11681878_14,10.1007/11787006_1}. 
% The compliance team answers the queries asked by the model developers using a differentially-private randomized mechanism. For such a mechanism to resolve the trade-off between utility and privacy, query responses should not deviate too much from the true values, to ensure that the perturbed answers can still be relevant to the model developers for fairness assessment. At the same time, query responses should not be too close to the true values, to prevent leakage of protected attributes. Differential privacy provides a tool for the compliance team to traverse this trade-off. The privacy parameter $\epsilon$ can be used to quantify the privacy risk posed by releasing a query. A sufficiently low $\epsilon$ limits the model developers’s ability to identify the protected  attribute of an individual. However, lowering $\epsilon$ reduces the utility of the query \cite{10.1561/0400000042}.

The notion of $\epsilon$-differential privacy was introduced in \cite{10.1007/11681878_14,10.1007/11787006_1}. 
% While various definitions can be found in the literature \cite{10.5555/1987260.1987294,Dwork2010PanPrivateSA}, this work focuses on the pure $\epsilon$-differential privacy and its relaxation, ($\epsilon,\delta$)-differential privacy.\cite{10.1145/1866739.1866758} provides an extensive survey on differential privacy. 
The definition of differential privacy used in this work focuses on keeping the protected attributes private. Because the model developers already have access to a portion of the test dataset $(X,Y)$, we define neighboring datasets as datasets that differ only on one individual's protected  attribute $A$. For $A,A' \in \{0,1\}^n$, $\mathcal{S}=(X,Y,A)$ and $ \mathcal{S}' =(X,Y,A')$ are neighboring if $\|A-A'\|_1=1$. Let $\mathcal{D}$ denote a universe of all possible datasets.
% Notice that simply returning random noise guarantees privacy but does not provide any useful information. On the other hand, returning fully accurate queries is a privacy violation since that might leak the protected attribute as illustrated in Section~\ref{setting}.  This leads us to propose the use of differential privacy. 

\begin{definition}
 [$(\epsilon,\delta)$-Differential privacy] Consider any two test datasets $\mathcal{S}=(X,Y,A)$ and $ \mathcal{S}' =(X,Y,A')$, where $A$ and $A'$ differ on the protected  attribute $A$ of one individual.
 We say that a randomized mechanism $\mathcal{M}$ is $(\epsilon,\delta) $-differentially private if, for all neighbouring
$\mathcal{S}$, $\mathcal{S}'$, and all $ \tau \subseteq Range(\mathcal{M})$, we have:\;\;
$$\Pr[\mathcal{M}(\mathcal{S})\in \tau] \leq e^\epsilon \Pr[\mathcal{M}(\mathcal{S}')\in \tau]+\delta \;\; \forall \mathcal{S}, \mathcal{S}'\in \mathcal{D},$$
where the randomness is over the choices made by $\mathcal{M}$ and $\epsilon >0$ is the privacy budget parameter. 
\end{definition}
% The choice of $\epsilon$ determines the amount of noise required. 
A smaller $\epsilon$ introduces greater noise, resulting in enhanced privacy but reduced output accuracy. On the other hand, a larger $\epsilon$ incorporates less noise, leading to weaker privacy guarantees but increased output accuracy. Here, $\delta$ is the probability of information being accidentally leaked. If $\delta$ = 0,  $\mathcal{M}$ is $\epsilon$-differentially private. A popular mechanism that achieves $\epsilon$-differential privacy is the Laplace mechanism \cite{10.1561/0400000042}. The Gaussian mechanism achieves $(\epsilon,\delta)$-DP for numeric queries (details in \cite[Theorem A.1]{10.1561/0400000042}).

\subsection{Laplace Mechanism for Answering Bias Queries Using Global Sensitivity}
\label{lapSec}

We first introduce the definition of global sensitivity for a set of bias queries, e.g., SP queries for a set of $m$ models.
\begin{definition}[$l_1 $-Global sensitivity \cite{10.1561/0400000042}]
\label{lpsen}The $l_1 $-sensitivity of a query function $f$ for all neighboring  $\mathcal{S}, \mathcal{S}' \in \mathcal{D}$ is: $$ \Delta_{f} = \displaystyle \max_{\mathcal{S},\mathcal{S}'}\|f(\mathcal{S})-f(\mathcal{S}')\|_1.$$
\end{definition}
% \noindent \textcolor{red}{For an $m$-dimensional query function $f(\cdot)$, a mechanism $\mathcal{M}$ that adds i.i.d.\ noise from a distribution $\textrm{Lap}({\Delta_{f}}/{\epsilon})$\footnote{$\textrm{Lap}(b)$ denotes a Laplace distribution with scale parameter $b$, i.e., $\textrm{Lap}(x|b)=\frac{1}{2b}\exp \big(-\frac{|x|}{b}\big)$. The variance $\sigma^2=2b^2$. } to each query is $\epsilon$-differential private \cite[Theorem 3.6]{10.1561/0400000042}.
%  $$\mathcal{M}(\mathcal{S},f(\cdot),h(\cdot),\epsilon)= f(\cdot)+(\sigma_1,\sigma_2,...,\sigma_m)$$
 
%  where $\sigma_i$ are i.i.d random variables from $\textrm{Lap}({\Delta_f}/{\epsilon})$. The Gaussian mechanism is used to achieve $(\epsilon,\delta)$-DP for numeric queries (details in \cite[Theorem A.1]{10.1561/0400000042}). SD: Are you repeating this again in Laplace Mechanism section? }
A naive differentially private technique the compliance team could employ is the Laplace mechanism. \\
\noindent \textbf{Laplace Mechanism:} 
Given a query function $f: \mathcal{D} \rightarrow \mathbb{R}^m$, the Laplace mechanism releases queries as follows: $$\mathcal{M}(\mathcal{S},f(\cdot),\epsilon)= f(\mathcal{S})+(n_1,n_2,...,n_j,...,n_m)$$
\noindent where $n_j$ are i.i.d. random variables drawn from $\textrm{Lap}({\Delta_{f}}/{\epsilon})$.
\begin{restatable}{theorem}{thmdelSP}\label{lapSP}
Given statistical parity gap queries ($SP$) for $m$ models, the Laplace mechanism that adds noise from $\rm{Lap}({\Delta_{SP}}/{\epsilon})$ to each query is $\epsilon$-differential private, where $\Delta_{SP}= \frac{m}{2}+\frac{m}{n-1}.$
\end{restatable}

\begin{restatable}{theorem}{thmabsSP}\label{ASPD}
Given absolute statistical parity gap queries ($|SP|$) for $m$ models, the Laplace mechanism that adds noise from  $\rm{Lap}({\Delta_{SP}}/{\epsilon})$ to each query is $\epsilon$-differential private, where $\Delta_{|SP|}= \frac{m}{2}.$
\end{restatable}

%\begin{proof} 

% To prove this, we find the $l_1$-global sensitivity defined in \ref{lpsen} for statistical parity gap $SP$. We assume that $N_0<N_1$, and there is at least $1$ individual in the disadvantaged group, i.e., $N_0 \geq 1$ in all possible datasets. A neighboring dataset is one such that the protected attribute differs for only one individual. There are two possibilities to consider here.

Similarly, for equal opportunity gap $EO$ and absolute equal opportunity gap $|EO|$ the Laplace mechanism adds noise with sensitivity $\Delta_{EO}=
\frac{m}{2}+\frac{m}{n-1}$ and $\Delta_{|EO|}= \frac{m}{2}$. See Appendix \ref{sasd} for proofs.

% The proofs of Theorem ~\ref{ASPD} and ~\ref{eolap} follows similar techniques as Theorem~\ref{lapSP}, and will be included in an expanded version or appendix. 

% The fairness metrics used are bounded between $[-1,1]$, and $[0,1]$ for their absolute values.
\begin{remark}Because the Laplace and Gaussian mechanisms have infinite support $(-\infty,\infty)$, query results can sit outside the range of our fairness metrics  ($[-1,1]$, or $[0,1]$ for absolute value metrics). In general, these mechanisms do not automatically deal with bounding constraints. Some choose to ignore them and release the raw outputs of the mechanisms since it still satisfies DP's privacy and accuracy guarantees. In our case, probabilities of out-of-bounds values are often small unless $\epsilon$ is chosen to be very small. If one insists on having bounded outputs, there are recent approaches~\cite{liu2019statistical}, such as the truncated and boundary-inflated truncation approaches. Other approaches map out-of-bounds outputs to the boundaries of the metric.
% In the truncation procedure, the out-of-bound output values are thrown away and one samples continuously until an output within the bounds is attained. This is equivalent to sampling from a truncated distribution that satisfies DP. Another approach is to map the out-of-bounds outputs to the boundaries of the metric. \textcolor{red}{Can make this shorter while still showing we are scholarly}
\end{remark}

\subsection{Attribute-Conceal: Our Proposed Technique Using Smooth Sensitivity} 
\label{SmoSen}
We have focused on adding noise to the query calibrated to its global sensitivity.  However, this might be excessive in many cases, that is,  the frameworks  would add so much noise that the output would be meaningless. Since we are interested in a particular test dataset $\mathcal{S}$,  we define the local sensitivity of a query function $f$ and test dataset $\mathcal{S}$ in $l_1$ as:
\begin{equation}\label{eqn:localsen}
 \Delta^{local}_f(\mathcal{S})= \max_{\mathcal{S'}:d(S,S')=1}\|f(\mathcal{S})-f(\mathcal{S'})\|_1.
\end{equation}
The challenge of calibrating noise to the local sensitivity $\Delta^{local}_f(\mathcal{S})$ is that it might  leak information about the test dataset and therefore not sufficient to guarantee DP~\cite{10.1145/1250790.1250803}. To address this, we investigate the idea of smooth sensitivity introduced in \cite{10.1145/1250790.1250803}. This is an intermediate notion between local and global sensitivity that allows dataset-specific additive noise to be added to achieve DP.
\begin{algorithm}[t]\caption{Attribute-Conceal for Compliance Team}\label{algo:conceal}
\begin{algorithmic}
\State \textbf{Input:} Model predictions $h_1(\cdot)$, $h_2(\cdot)$, \ldots, $h_m(\cdot)$, $N_1$, $N_0$, $\epsilon$
\State Compute statistical parity gap for each model and store in $\overline{SP}$
\State Compute Smooth Sensitivity: $$
\Delta^{smooth}_{SP,\beta}(\mathcal{S})= \max \bigg( \frac{m}{N_1+1}+\frac{m}{N_0}\; ,\; e^{\frac{-\epsilon(N_0-2)}{6m}}\left(\frac{m}{n-1}+\frac{m}{2}\right) \bigg)
$$
\For{$i=1,2,\ldots,m$}
\State Sample $Z_i$ from a standard Cauchy distribution
\State $\overline{SP}_i \gets  \overline{SP}_i + \frac{6 \Delta^{smooth}_{SP,\beta}}{\epsilon} Z_i $\Comment{This algorithm is based on Theorem \ref{caucy}}
\EndFor 
\State \textbf{Return:} $\overline{SP}$
\end{algorithmic}
\end{algorithm}
\begin{definition}[Smooth sensitivity \cite{10.1145/1250790.1250803}] 
\label{SSens} For $\beta>0$, the $\beta$-smooth sensitivity of f is 
\begin{equation}\Delta^{smooth}_{f,\beta}(\mathcal{S})=\max_{\mathcal{\bar{S}} \in \mathcal{D} } \bigg(\Delta^{local}_f(\mathcal{\bar{S}}) \cdot e^{-\beta d(\mathcal{S},\mathcal{\bar{S}})}\bigg),
\end{equation}\label{def:SS}
where $d(\mathcal{S},\mathcal{\bar{S}})$ denotes the number of entries in which protected attribute vectors $A$ and $\bar{A}$ disagree.
\end{definition}
\begin{figure*}
     \centering
     \begin{subfigure}{0.32\textwidth}
         \centering
      \includegraphics[scale=0.165]{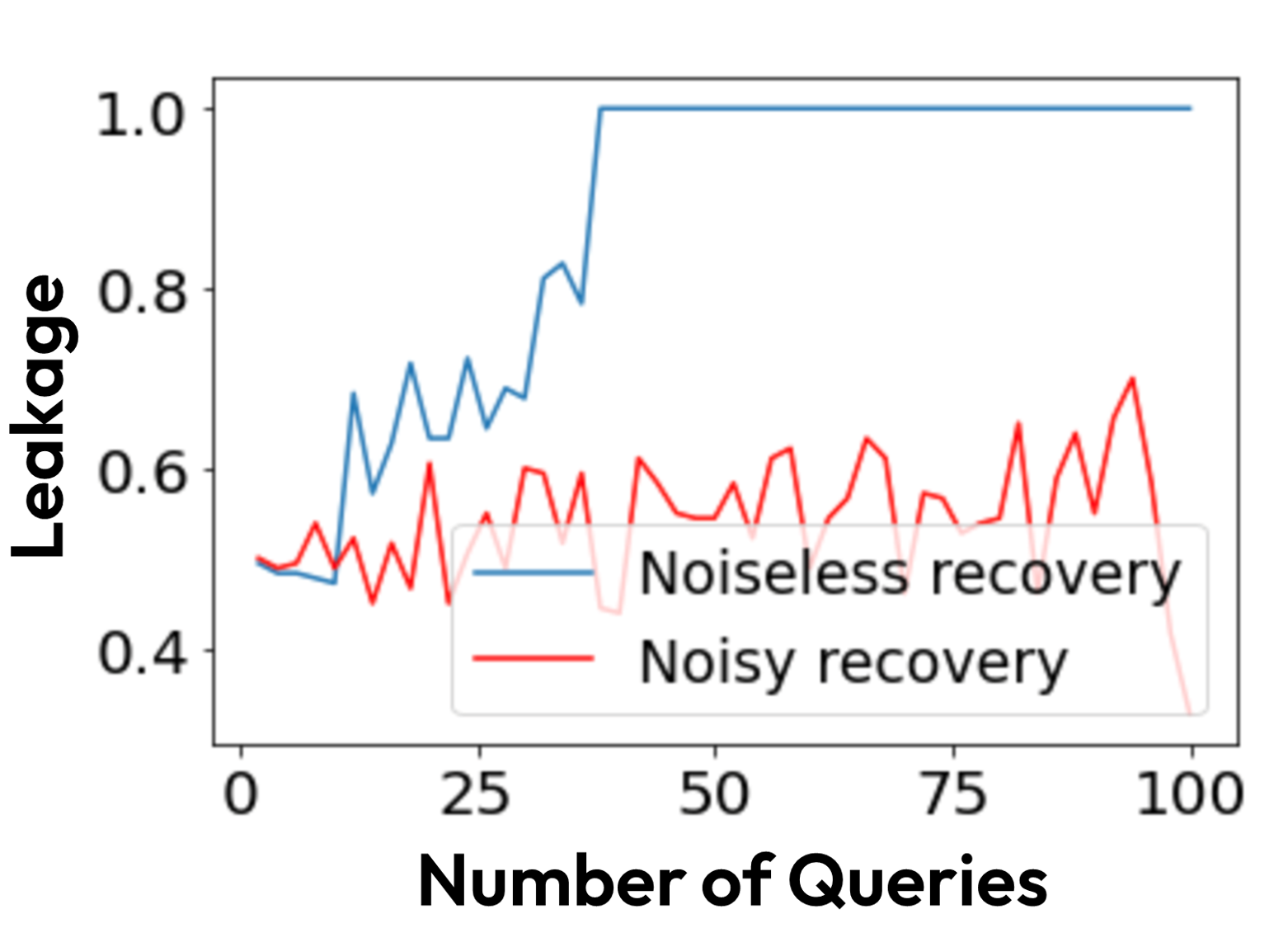}
     %{ScreenShot2022-09-23at3_31_17PM.png}
     \end{subfigure}
     \hfill
     \begin{subfigure}{0.32\textwidth}
         \centering
         \includegraphics[scale=0.165]{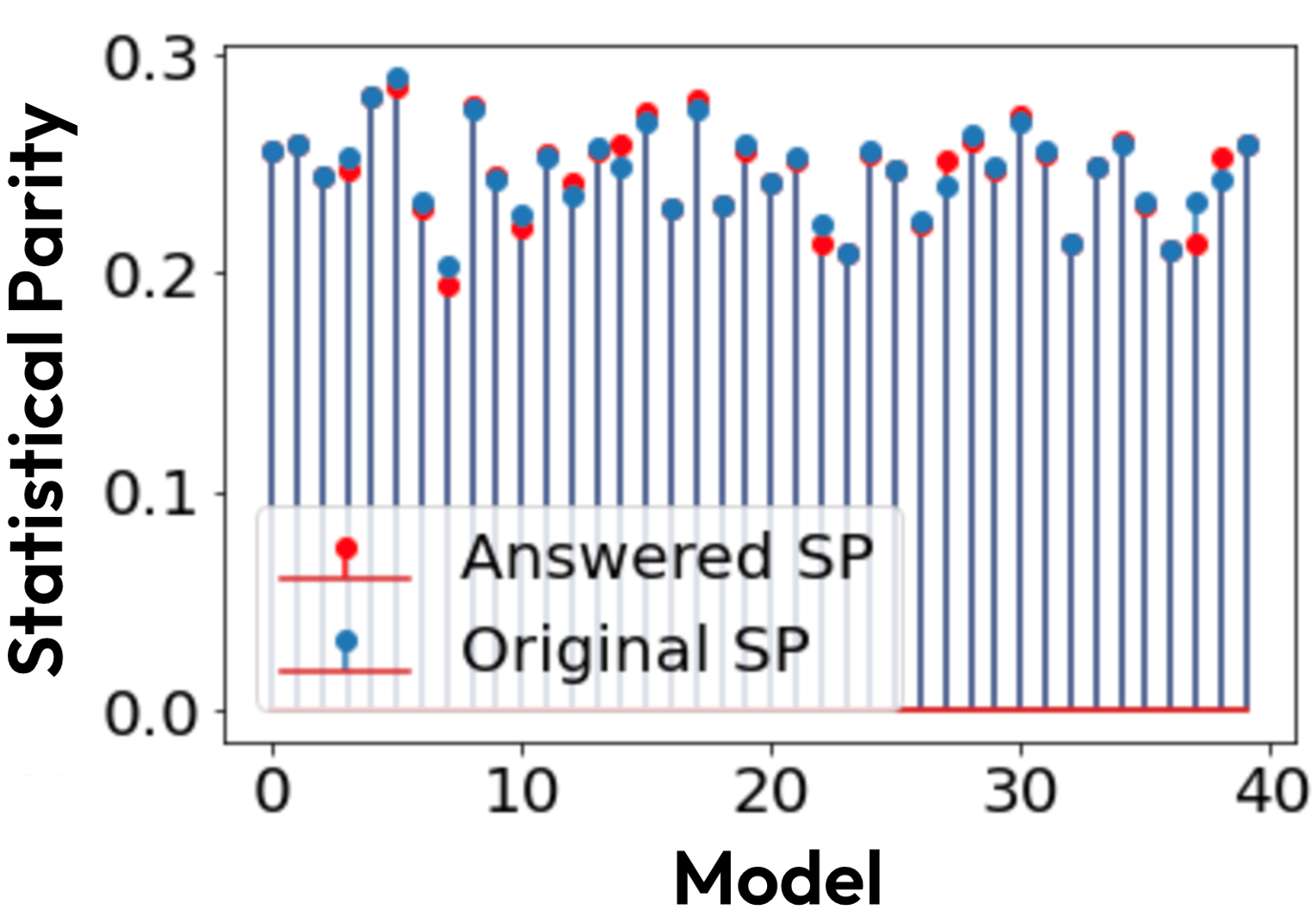}

     \end{subfigure}
     \hfill
     \begin{subfigure}{0.32\textwidth}
         \centering
         \includegraphics[scale=0.165]{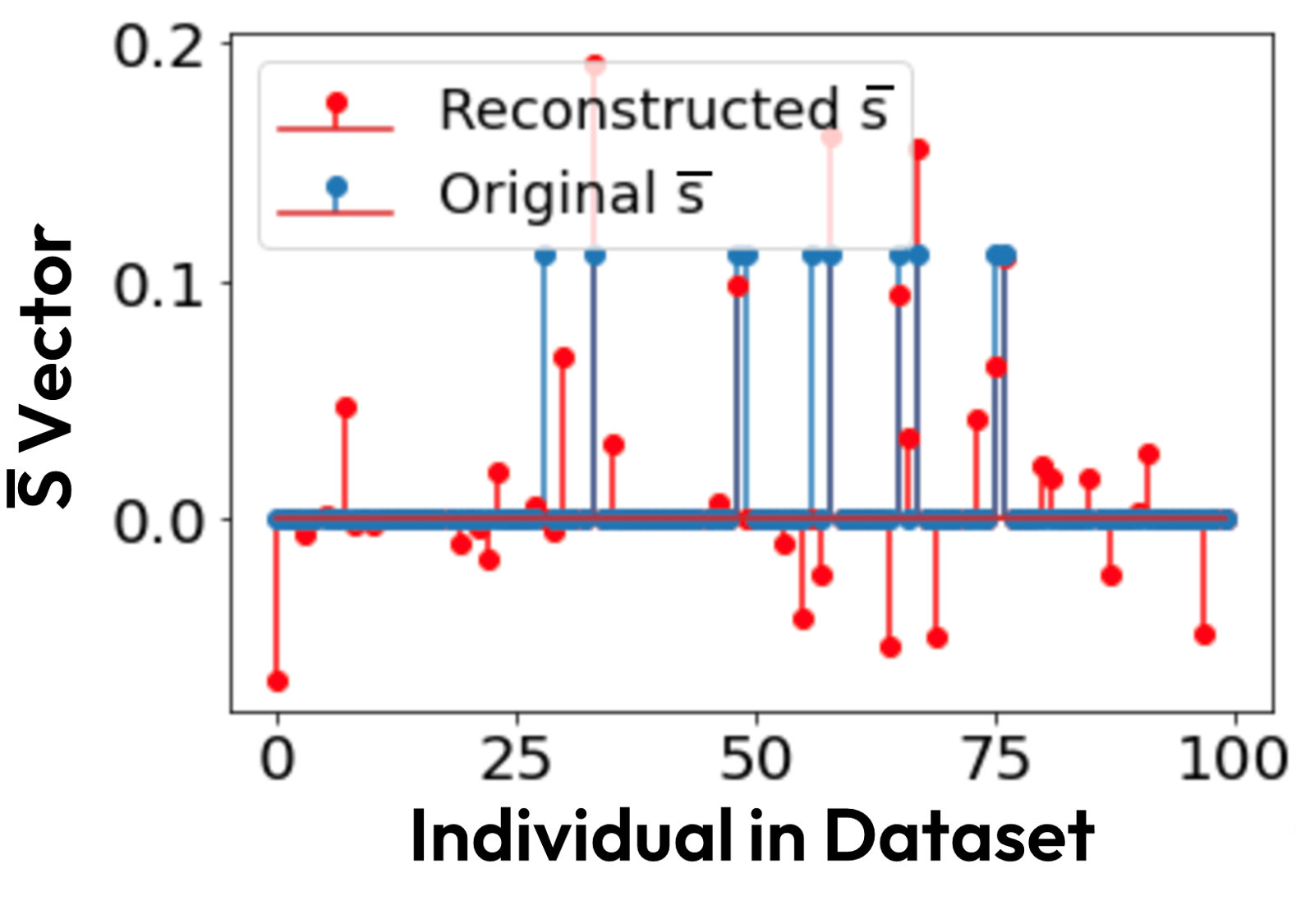}

     \end{subfigure}
        \caption{Experimental Results on Adult dataset for test size $n=100$: (a) Leakage as a function of No. of queries in noisy ($\epsilon=100$) and noiseless case ($\epsilon= \infty$). Attribute-Conceal prevents Leakage even with an increase in the number of queries. Note that random guessing achieves a Leakage of about 50\%, meaning no individual's protected attribute is recovered with certainty. (b) Answered SP queries for $m=40$ achieves a low Avg. SP err of $7.41 \times  10^{-4}$. (c) Reconstructed $\bar{s}$ vector with Attribute-Conceal varies a lot from the original ($\bar s$ vector reveals an individual's protected attribute, see equation \eqref{cseq}).  Trained models have an Avg. accuracy of 86.23\%, and std of 0.2583.}
        \label{erssqe}
\end{figure*}

\begin{table*}[h]
\caption{Detailed Experimental Results on Adult dataset for test size $n=100$ and $n=1000$: Attribute-Conceal (Ours) has much lower Avg. SP err (query error) than Laplace Mechanism (Lap.) for the same privacy parameter $\epsilon$ (and similar Leakage).}
\begin{tabular}{|c|c|cc|cc|c|c|cc|cc|}
\hline
{$m$} & {$\epsilon$} & \multicolumn{4}{c|}{$n=100$} & {$m$} & {$\epsilon$} & \multicolumn{4}{c|}{$n=1000$} \\ \cline{3-6} \cline{9-12}
& & \multicolumn{2}{c|}{Avg. SP err ($\times 10^{-3}$)} & \multicolumn{2}{c|}{Leakage (\%)} & & & \multicolumn{2}{c|}{Avg. SP err ($\times 10^{-3}$)} & \multicolumn{2}{c|}{Leakage (\%)} \\ \cline{3-6} \cline{9-12}
& & Ours & Lap. & Ours & Lap. & & & Ours & Lap. & Ours & Lap. \\ \hline
{$25$} & 5 & 10.9 & 200.9 & 50 & 54 & {$300$} & 5 & 122.2 & 874.5 & 52 & 49 \\ 
& 10 & 5.1 & 112.6 & 48 & 49 & & 10 & 35.7 & 160.8 & 50 & 51 \\ 
& 100 & 0.05 & 12.1 & 58 & 44 & & 100 & 2.0 & 11.5 & 53 & 51 \\ 
& $\infty$ & 0 & 0 & 63 & 63 & & $\infty$ & 0 & 0 & 79 & 79 \\ \cline{1-12} 
{$40$}
& 5 & 77.9 & 660.2 & 55 & 57 & {$400$} & 5 & 140.0 & 893.4 & 52 & 49 \\ 
& 10 & 12.3 & 236.8 & 49 & 43 & & 10 & 42.2 & 216.3 & 52 & 52 \\ 
& 100 & 0.7 & 27.6 & 67 & 47 & & 100 & 0.2 & 54.4 & 55 & 49 \\ 
& $\infty$ & 0 & 0 & 100 & 100 & & $\infty$ & 0 & 0 & 100 & 100 \\ \hline
\end{tabular}
\label{uuiuhui}
\end{table*}

\begin{theorem}[Dataset Specific $\epsilon$-DP Statistical Parity Query]
\label{caucy} 
% \rd{VERSION 1}
% Let $Z = (z_1,z_2,...,z_m)$ be an $m$-dimensional random noise vector with entries independently sampled from a Cauchy distribution with density proportional to $\frac{1}{1+z_i^2}$. For statistical parity gap query ($SP$), the mechanism $\mathcal{M}(\mathcal{S})= SP(\mathcal{S})+\frac{6 \Delta^{smooth}_{SP,\beta}(\mathcal{S})}{\epsilon} \cdot Z$ is $\epsilon$-differentially private, where the smooth sensitivity $\Delta^{smooth}_{SP,\beta}(\mathcal{S})$ is given by:
% \\
% \rd{VERSION 2} 

Let $Z$ be an $m$-dimensional random noise with entries independently sampled from a Cauchy distribution $P(z)= \prod_{i=1}^m\frac{1}{1+z_i^2}$. For statistical parity gap query ($SP$), the mechanism $\mathcal{M}(\mathcal{S})= SP(\mathcal{S})+\frac{6 \Delta^{smooth}_{SP,\beta}(\mathcal{S})}{\epsilon}Z$ is $\epsilon$-differentially private, where the smooth sensitivity $\Delta^{smooth}_{SP,\beta}(\mathcal{S})$ is given by:

% Let $Z$ be random noise samples from an $m$-dimensional Cauchy distribution with  density proportional to $\frac{1}{1+z^2}$. For statistical parity gap query ($SP$), the mechanism $\mathcal{M}(\mathcal{S})= SP(\mathcal{S})+\frac{6 \Delta^{smooth}_{SP,\beta}}{\epsilon} \cdot Z$ is $\epsilon$-differentially private, where the smooth sensitivity $\Delta^{smooth}_{SP,\beta}$ is given by:
$$
\Delta^{smooth}_{SP,\beta}(\mathcal{S})= \max \bigg( \frac{m}{N_1+1}+\frac{m}{N_0}\; ,\; e^{\frac{-\epsilon(N_0-2)}{6m}}\left(\frac{m}{n-1}+\frac{m}{2}\right) \bigg).
$$
Here, $N_0$ and $N_1$ are sizes of disadvantaged and advantaged groups in dataset $\mathcal{S}$, such that $N_0 \leq N_1$, $N_0+N_1=n$, and $\beta = \frac{\epsilon}{6m}$.
\end{theorem}
\begin{proof} We first find the $\beta$-smooth sensitivity of statistical parity gap $\Delta^{smooth}_{SP,\beta}(\mathcal{S})$.
\begin{align*}
\Delta^{smooth}_{SP,\beta}(\mathcal{S})&  {\stackrel{\text{(a)}}{=}} \max_{\mathcal{\bar{S}} \in \mathcal{D} } \bigg(\Delta LS_{SP}(\mathcal{\bar{S}}) {\cdot} e^{-\beta d(\mathcal{S},\mathcal{\bar{S}})}\bigg)\\
& {\stackrel{\text{(b)}}{=}} \max_{k=0,1,\ldots,N_0-2} e^{-k\beta} \bigg( \max_{\bar{S}:d(S,\bar{S})=k} \Delta LS_{SP}(\bar{S})\bigg)\\
 & {\stackrel{\text{(c)}}{=}} \max_{k=0,\ldots,N_0-2} e^{-k\beta} \bigg( \frac{m}{N_1+k+1}{+}\frac{m}{N_0-k}\bigg)\\
 & {\stackrel{\text{(d)}}{=}} \max \bigg( \frac{m}{N_1+1}{+}\frac{m}{N_0},e^{-(N_0{-}2)\beta}\bigg(\frac{m}{n-1}{+}\frac{m}{2}\bigg)\bigg).
\end{align*}
% (for details \cite[Def 3.1]{10.1145/1250790.1250803}. 
% Here, $(b)$ is the expression of the smooth sensitivity in terms of the $k$ distance sensitivity of $SP$ \cite[Def 3.1]{10.1145/1250790.1250803}. 
% The $k$ distance sensitivity is the maximum change in $SP$ when considering neighboring datasets that are distance $k$ away (note that local sensitivity in \eqref{eqn:localsen} is the  $k=1$ distance sensitivity). 
Here, (a) is from the definition of smooth sensitivity in \eqref{def:SS}. Next, (b) is the expression of the smooth sensitivity when looking at datasets at distance $k$ (for details see \cite[Def 3.1]{10.1145/1250790.1250803}). To obtain (c) we find the maximum local sensitivity over datasets that are at distance $k$ (see Lemma \ref{lem:kdist} in Appendix~\ref{derivative}). Finally, (d) holds since the function is convex and hence its maximum occurs at the boundaries $k \in \{0,N_0-2\}$ (see Lemma \ref{lem:maxb} in Appendix~\ref{derivative}).
The rest of the proof follows directly from Lemma \ref{lemma26} and Lemma \ref{lemma27} in Appendix \ref{apx:ss} with $\gamma=2$ and $\beta=\frac{\epsilon}{6m}$.
\end{proof}

To achieve pure differential privacy, noise is introduced following a Cauchy distribution. This motivates Algorithm~\ref{algo:conceal} (Attribute-Conceal), a differentially private technique to answer statistical parity gap queries based on Theorem~\ref{caucy}. 

We also note that the behavior of the Cauchy distribution can sometimes be unusual, as it does not have an expected value and has heavy tails that decay polynomially, compared to the exponential decay observed in Laplace and Gaussian distributions. In Theorem \ref{thm:lastlap}, we therefore also provide a relaxed $(\epsilon,\delta)$-differentially private mechanism that introduces noise from a Laplace distribution.
\begin{figure*}[t]
     \centering
     \begin{subfigure}[b]{0.32\textwidth}
         \centering
        \includegraphics[scale=0.17]{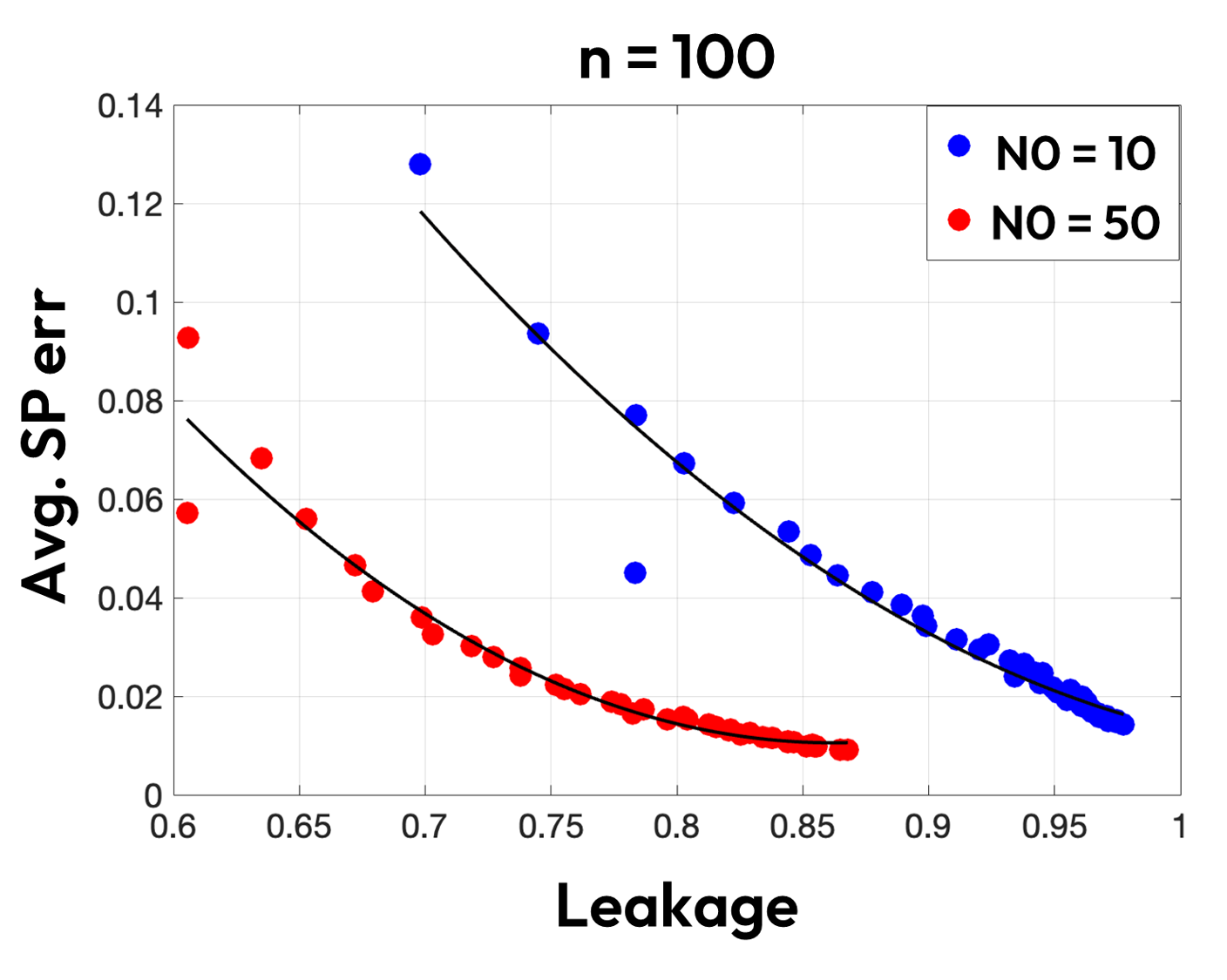}
    
         \label{fig:y equals x}
     \end{subfigure}
     \hfill
     \begin{subfigure}[b]{0.32\textwidth}
         \centering
         \includegraphics[scale=0.17]{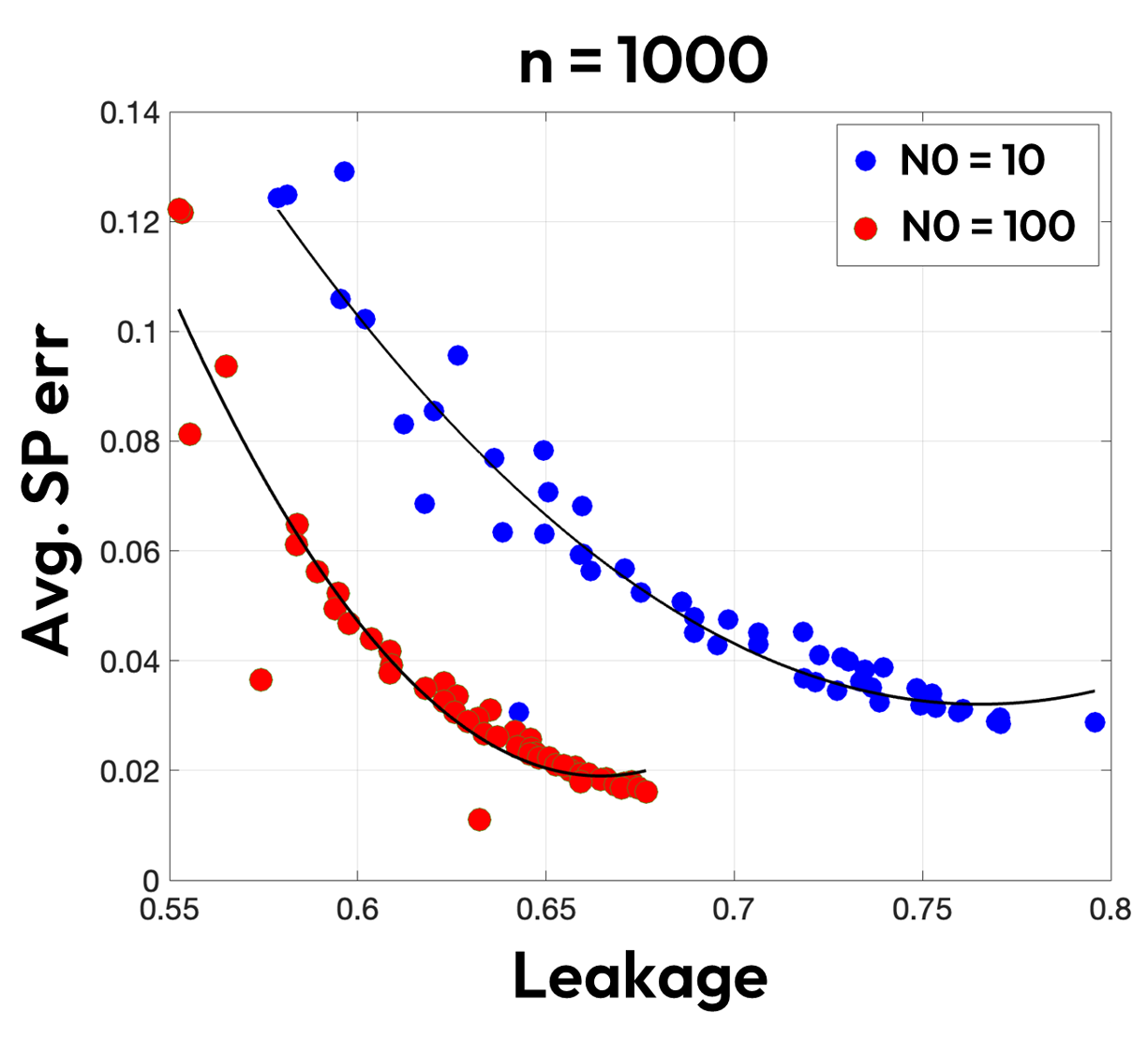}
         
         \label{fig:three sin x}
     \end{subfigure}
     \hfill
     \begin{subfigure}[b]{0.32\textwidth}
         \centering
         \includegraphics[scale=0.155]{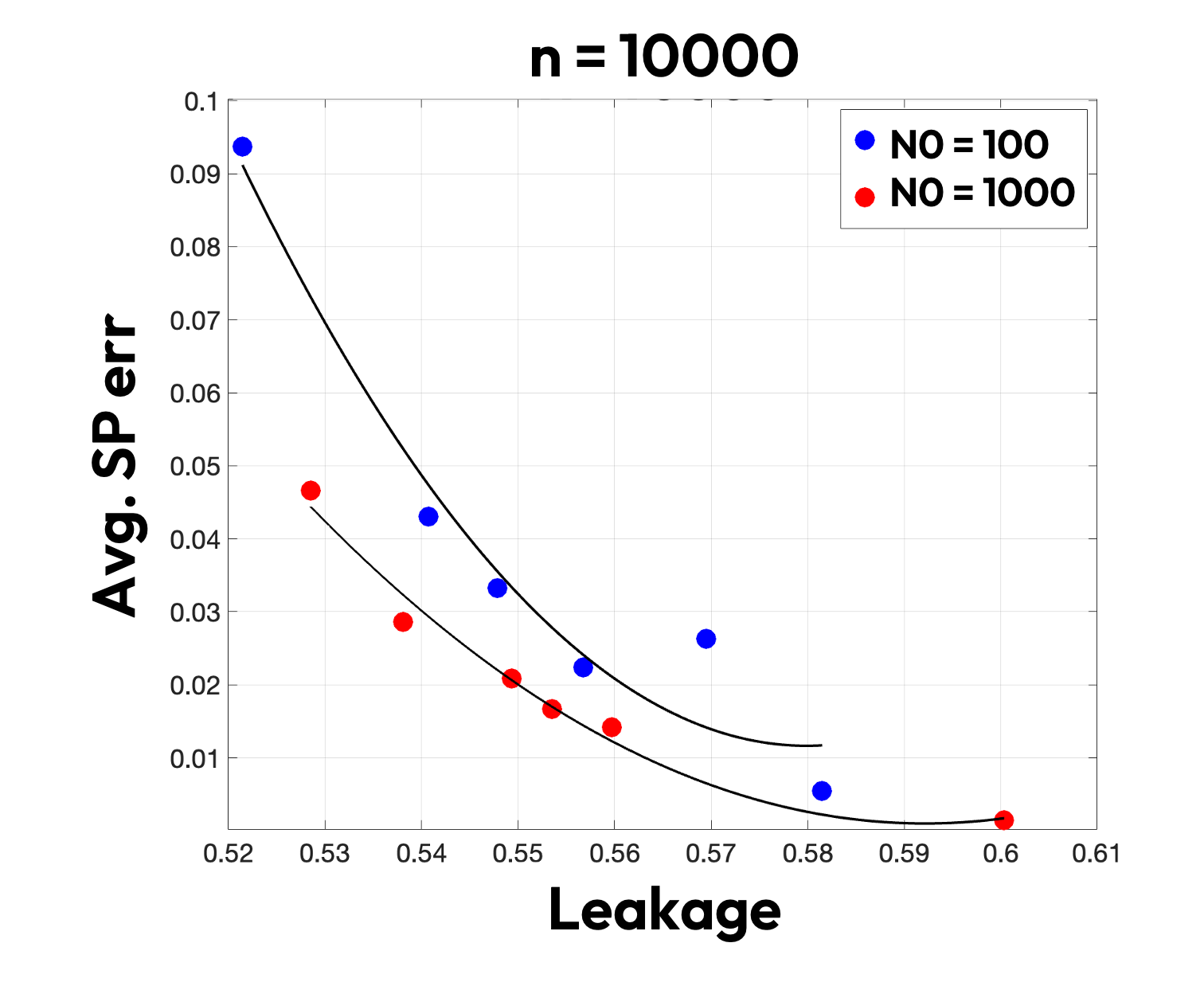}

     \end{subfigure}
        \caption{Experimental Results on Synthetic data for test size $n=100,1000,$ and $10000$: Avg. SP err and Leakage trade-off with Attribute-Conceal. Each point represents an $\epsilon \in [10,500]$ averaged over 50 runs. Results for varying sparsity $N_0$ and $m=\mathcal{O}( N_0 \log n /N_0$).}
        \label{erczczqe}
\end{figure*}

\begin{table*}[h]
\caption{Experimental Results on Synthetic data: Avg. SP err and Leakage for test dataset size $n=1000$ and $n=10000$ for Attribute-Conceal varying test dataset sparsity and model number $m$ with  $\epsilon=100$.}
\begin{tabular}{|c|cccc|c|cc|}
\hline
{$N_0$} & \multicolumn{4}{c|}{$n=1000$ $\epsilon=100$}                                                         & {$N_0$} & \multicolumn{2}{c|}{$n=10000$ $\epsilon=100$}               \\ \cline{2-5} \cline{7-8} 
                       & \multicolumn{2}{c|}{Avg. SP err ($\times 10^{-3}$)}                     & \multicolumn{2}{c|}{Leakage (\%)} &                        & \multicolumn{1}{c|}{Avg. SP err ($\times 10^{-3}$)} & Leakage (\%) \\ \cline{2-5} \cline{7-8} 
                       & \multicolumn{1}{c|}{$m=500$} & \multicolumn{1}{c|}{$m=900$} & \multicolumn{1}{c|}{$m=500$} & $m=900$ &                        & \multicolumn{1}{c|}{$m=1000$}           & $m=1000$          \\ \hline
10                     & \multicolumn{1}{c|}{128.8}  & \multicolumn{1}{c|}{191.4}  & \multicolumn{1}{c|}{65}    & 64    & 50                     & \multicolumn{1}{c|}{96.9}             & 64              \\
50                     & \multicolumn{1}{c|}{90.8}  & \multicolumn{1}{c|}{156.3}  & \multicolumn{1}{c|}{67}    & 61    & 100                    & \multicolumn{1}{c|}{76.9}             & 61              \\
100                    & \multicolumn{1}{c|}{87.4}  & \multicolumn{1}{c|}{102.9}  & \multicolumn{1}{c|}{62}    & 55    & 500                    & \multicolumn{1}{c|}{38.3}             & 59              \\
200                    & \multicolumn{1}{c|}{51.9}  & \multicolumn{1}{c|}{96.7}  & \multicolumn{1}{c|}{62}    & 47    & 1000                   & \multicolumn{1}{c|}{26.9}             & 59              \\
300                    & \multicolumn{1}{c|}{38.9}  & \multicolumn{1}{c|}{59.9}  & \multicolumn{1}{c|}{48}    & 47    & 3000                   & \multicolumn{1}{c|}{12.8}             & 55              \\
400                    & \multicolumn{1}{c|}{41.6}  & \multicolumn{1}{c|}{54.5}  & \multicolumn{1}{c|}{47}    & 45    & 5000                   & \multicolumn{1}{c|}{9.9}             & 48              \\ \hline
\end{tabular}
 \label{table31}
\end{table*}

\begin{table*}[h]
\caption{Experimental Results on Synthetic data: Avg. SP err and Leakage for test size $n=1000$ and $n=1000$ for Attribute-Conceal varying $N_0$ and privacy parameter $\epsilon$.}
\begin{tabular}{|c|ccccccc|cccc|}
\hline
                            & \multicolumn{7}{c|}{$n=1000$}                                                                                     & \multicolumn{4}{c|}{$n=10000$}                                                                                 \\ \cline{2-12} 
{$\epsilon$} & \multicolumn{4}{c|}{Avg. SP err ($\times 10^{-3}$)}                        & \multicolumn{3}{c|}{Leakage (\%)}           & \multicolumn{2}{c|}{Avg. SP err ($\times 10^{-3}$)}                          & \multicolumn{2}{c|}{Leakage (\%)}      \\ \cline{2-12} 
                            & \multicolumn{2}{c|}{$N_0=10$} & \multicolumn{2}{c|}{$N_0=100$} & \multicolumn{2}{c|}{$N_0=10$} & $N_0=100$        & \multicolumn{1}{c|}{$N_0=100$} & \multicolumn{1}{c|}{$N_0=10^3$} & \multicolumn{1}{c|}{$N_0=100$} & $N_0=10^3$ \\ \hline
10                          & \multicolumn{2}{c|}{929}    & \multicolumn{2}{c|}{838}     & \multicolumn{2}{c|}{49}     & 51             & \multicolumn{1}{c|}{832.3}    & \multicolumn{1}{c|}{678.4}     & \multicolumn{1}{c|}{51}      & 50       \\
50                          & \multicolumn{2}{c|}{149.3}   & \multicolumn{2}{c|}{44.8}    & \multicolumn{2}{c|}{65}     & 55             & \multicolumn{1}{c|}{46.5}    & \multicolumn{1}{c|}{85.3}     & \multicolumn{1}{c|}{55}      & 52       \\
100                         & \multicolumn{2}{c|}{83}    & \multicolumn{2}{c|}{30.5}    & \multicolumn{2}{c|}{71}     & 56             & \multicolumn{1}{c|}{33.1}    & \multicolumn{1}{c|}{47.4}     & \multicolumn{1}{c|}{55}      & 53       \\
500                         & \multicolumn{2}{c|}{26.6}   & \multicolumn{2}{c|}{9.7}    & \multicolumn{2}{c|}{75}     & 53             & \multicolumn{1}{c|}{31}     & \multicolumn{1}{c|}{13.6}     & \multicolumn{1}{c|}{57}      & 56       \\ \hline
\end{tabular}
 \label{table3}
\end{table*}

\begin{theorem}\label{thm:lastlap} Let $Z$ be random noise samples from a $m$-dimensional Laplace distribution $P(z)=\frac{1}{2^m}  e^{-\|z\|_1}$. For statistical parity gap query ($SP$), and $ \epsilon,\delta \in (0,1)$, the mechanism $\mathcal{M}(\mathcal{S})= SP(\mathcal{S})+\frac{2 \Delta^{smooth}_{SP,\beta}}{\epsilon}Z$ is $(\epsilon,\delta)$-differentially private, where the smooth sensitivity is given by:
$$
\Delta^{smooth}_{SP,\beta}(\mathcal{S})= \max \bigg( \frac{m}{N_1+1}+\frac{m}{N_0}\; ,\; e^{-\frac{(N_0-2)\epsilon}{4(m+\ln(2/\delta))}}\left(\frac{m}{n-1}+\frac{m}{2}\right) \bigg).
$$
Here, $N_0$ and $N_1$ are sizes of disadvantaged and advantaged groups in dataset $\mathcal{S}$, such that $N_0 \leq N_1$, $N_0+N_1=n$, and $\beta = \frac{\epsilon}{4(m+\ln(2/\delta))}$.
\end{theorem}
\begin{proof}
The proof follows from Lemma \ref{lemma26} and \ref{lemma29} in Appendix \ref{apx:ss}. Sensitivity analysis follow from proof of Theorem \ref{caucy} with $\beta = \frac{\epsilon}{4(m+\ln(2/\delta))}$
\end{proof}

These results can be extended to the absolute statistical parity gap with $\beta$-smooth sensitivity (see Appendix~\ref{AAS}), i.e.,
$$\Delta^{smooth}_{|SP|,\beta}(\mathcal{S})=\max \bigg(\frac{m}{N_0} , \frac{me^{-(N_0-2)\beta}}{2} \bigg).$$

\section{Experiments}
We include experimental results on the Adult dataset (see Table~\ref{uuiuhui} and Figure~\ref{erssqe}) and simulations on synthetic dataset (see Figure~\ref{erczczqe}, Table~\ref{table31} and Table~\ref{table3}). For the Adult dataset, the protected attribute is race (assumed binary). We restrict ourselves to only White and Black with the latter being relatively sparse (10.4\%). We first demonstrate how querying using Attribute-Reveal can leak the protected attributes. Then, we show that Attribute-Conceal effectively prevents this leakage (also outperforming the naive Laplace mechanism). Our performance metrics of interest are (1) Average error in answering the Statistical Parity query (Avg. SP Err); and (2) Accuracy of correctly recovering (essentially leaking) the protected attribute balanced across both races (Leakage, formally defined in Definition~\ref{leak}). To observe the tradeoff between Avg. SP Err (query error) and Leakage over a broader range of parameters (privacy parameter $\epsilon$, sparsity $N_0$, and test size $n$), we also perform simulations on a synthetic dataset. We provide additional experimental results on the German Credit dataset~\cite{Dua:2019a} in Appendix~\ref{expref}.

\begin{definition}[Leakage(\%)]
\label{leak} Let ${N_A}$ be number of individuals in the advantaged group whose protected attribute was correctly predicted and ${N_B}$ be number of individuals in the disadvantaged group whose protected attribute was correctly predicted.\\  The leakage is defined as:
$$\text{Leakage}= \frac{1}{2} \left( \frac{{N_A}}{{N_1 }}+\frac{{N_B}}{{N_0}}\right) \times 100.$$
\end{definition}
The leakage is the balanced accuracy of recovery. This is used to deal with imbalanced data, i.e., when one target class appears a lot more than the other.

\subsection{Experiments with Adult Dataset}
The Adult dataset has 14 attributes for 48842 loan applicants. The classification task is to predict whether an individual's income is more or less that $50$K~\cite{1251_2016}. The feature "race" is chosen as the protected attribute. This feature is excluded from training and only used for statistical parity evaluation. We restrict ourselves to only White and Black (binary) with the latter being relatively sparse (10.4\%). We compare  Attribute-Conceal with a naive differential privacy technique, Laplace mechanism. We experiment with different test sizes and show our results in Figure \ref{erssqe} and Table \ref{uuiuhui}.

 Given an input, our base model $h_0(\cdot)$ outputs a probability value between $0$ and $1$. For the other $m$ models, we add a small noise sampled from Uniform$(-0.1,0.1)$ distribution to each output of the base model. 

We observe that the accuracy of the other models is quite close to the original. We created 40 models from the base model: they had a mean accuracy of $86.23\%$ and a standard deviation of $0.2583$.  

Interestingly, our experiments demonstrate that with the uniform noise, we can still recover the protected attributes with far fewer models than the full-rank case.  As shown in Figure \ref{erssqe}, we are able to recover all the protected attributes using $m=40$ models. Notice that, this is roughly $\mathcal{O}(N_0 \log(n/N_0))$. 

Our recovery of protected attributes is based on the values of the $\bar{s}$ vector in Algorithm~\ref{algo:reveal_cs}. Ideally, it should be $0$ if $a_j=1$ and ${1}/{N_1}+1/N_0$ if $a_j=0$. In our practical implementation, the compressed sensing solution is not always exact but still good enough to infer the protected attribute. Due to this, we use a threshold between $0$ and ${1}/{N_1}+1/N_0$ to identify the protected attribute.
\subsection{Experiments with Synthetic Dataset}
 We perform simulations on synthetic data to observe the trade-off between Avg. SP Err and Leakage over a broader range of parameters (privacy parameter $\epsilon$, sparsity $N_0$, and test size $n$). In Figure \ref{erczczqe}, we show this trade-off with Attribute-Conceal for test size $n=100,1000,$ and $10000$. Each point represents an $\epsilon \in [10,500]$ averaged over 50 runs. We show results for varying sparsity $N_0$ and $m=\mathcal{O}(N_0 \log n /N_0$). Table~\ref{table31} and Table~\ref{table3} provide additional experimental results highlighting the Avg. SP err and Leakage for test size $n=1000$ and $n=10000$ for different sparsity $N_0$, the model number $m$, and the privacy parameter $\epsilon$. A clear trend observed is that Attribute-Conceal results in a significantly lower Avg. SP error compared to the Laplace mechanism, for a similar level of protected attribute leakage.
\section{Conclusion and Future Work}
This work highlights a major concern with fairness assessments in scenarios where protected attributes such as gender or race cannot be accessed during model training. Showing that simply querying for fairness metrics can leak sensitive information to model developers raises important questions about the ethical implications of these assessments. As a remedy, we also propose a novel technique, Attribute-Conceal, which achieves differential privacy by calibrating noise to the smooth
sensitivity of our bias query. 

The results of this study have important implications for regulations and privacy in the field of algorithmic fairness and provide a new approach to protect the sensitive information of individuals in fairness assessments. This also provides a potential resolution to the continuing debate about whether protected attributes should be used in training. Future research could look into expanding the framework to include other fairness metrics or incorporating these techniques into training or post-processing to directly reduce bias without leaking protected attributes.

Our current approach assumes that both model developers and the compliance (or auditing) team work with the same test set. However, this  might not hold true in every context. The compliance/auditing team may choose to use a different test set. However, note that a different test set may not adequately represent the true training distribution, which could potentially affect generalization. 

We note that while our focus is on leakage from bias queries, future work could also look into \emph{inferring} the protected attributes from the other available features using alternate techniques~\cite{du2012privacy,7282765}. For example, if one has prior knowledge that a feature such as hours-worked-per-week is strongly correlated with gender, one might just be able to \emph{infer} gender with reasonable accuracy from that feature. However, it remains debatable if such \emph{indirect inferring} of protected attribute from correlated features would legally constitute a violation of disparate treatment (or privacy). On the other hand, asking the compliance team for bias assessments actually accesses the protected attributes using queries. We do make a distinction between \emph{leaking} and \emph{inferring} protected attributes here. An interesting scenario would arise if one exploits a synergy of both bias queries as well as inference mechanisms to obtain even more accurate predictions of protected attributes than using either of them individually, and if such techniques would constitute a violation of anti-discrimination and privacy.

\bibliographystyle{ACM-Reference-Format}

\bibliography{sample-base}

\medskip

%%%%%%%%%%%%%%%%%%%%%%%%%%%%%%%%%%%%%%%%%%%%%%%%%%%%%%%%%%%%

%%%%%%%%%%%%%%%%%%%%%%%%%%%%%%%%%%%%%%%%%%%%%%%%%%%%%%%%%%%%

\newpage
\appendix

\onecolumn

%\section{Appendix}

\section{Appendix to Section~\ref{SubSA}}
\label{1qw}

\subsection{Proof of Theorem~\ref{Recovering Sensitive Attributes $A$ using Linear Equations}}

\thrmone*
% \begin{equation}
%      \begin{pmatrix} SP_{1} \\ SP_{2} \\ . \\ . \\ . \\ SP_{m} \end{pmatrix}
%  =
%   \begin{pmatrix}
%   h_1(x_1) & h_1(x_2) & ... &...& h_1(x_n)\\
%   h_2(x_1) & h_2(x_2) & ... &...& h_2(x_n)
%   \\.& . & &&.\\.&  &. &&.\\.&  & &.&.\\
%   h_m(x_1) & h_m(x_2) & ... &...&h_m(x_n)
%   \end{pmatrix}
%   \begin{pmatrix} v_{1} \\ v_{2} \\ . \\ . \\ . \\ v_{n} \end{pmatrix} 
% \end{equation}

% and $\textrm{\textbf{\textup v}} $ is the unknown vector with elements taking values:
% $$v_j= 
% \begin{cases}
%         \frac{ 1}{N_1},& \text{\textit{if } } a_j=1\\
%     -\frac{1}{N_0}, & \text{\textit{if } } a_j=0.
% \end{cases}
% $$

According to Theorem~\ref{Recovering Sensitive Attributes $A$ using Linear Equations}, if there are as many linearly independent model predictions $h_i(X),i \in [m]$ as individuals in the dataset, then with their corresponding statistical parity queries,  $\overline{SP}=\begin{bmatrix} SP_1 & SP_2 & ... & SP_m \\  \end{bmatrix}^T $, one can \emph{always} get the protected attributes of all individuals in the dataset by solving for $\textrm{\textbf{v}}$ in \eqref{lse}. The protected attribute $a_j$ of the $j$-th individual is revealed by the value of each entry $v_j$ of the vector $\textrm{\textbf{v}} $. If we assume the worst-case scenario, where the model developers can directly choose the output predictions of these models, then all they have to do is to use $m=n$ models, and have each model's output prediction to be linearly independent, making $\textit{rank}(\textbf{H})= n$. As an example, suppose you have $m=n$ models, and the $i$-th model accepts only the $j$-th individual for $i=j$ and rejects everyone else. 
\section{Appendix To Section~\ref{subsec:cs}}\label{apx:col1}

\subsection{Background: Relevant Results in Compressed Sensing}
To prove Lemma \ref{col:uni_rip}, we will use some  results from existing literature on RIP for compressed sensing (restated in Theorem \ref{trn:RIP}, Theorem \ref{trn:subRIP}, and Lemma \ref{trn:uni}).

\begin{definition}[\cite{baraniuk2008simple}]\label{trn:SCE}
Given a random matrix $\Phi \in \mathbb{R}^{m \times n}$ and $x \in \mathbb{R}^n$, we say that $\Phi$ is strongly concentrated around its expectation if
$$
\Pr\left(\left|\|\Phi x\|_2-\|x\|_2\right|\geq\epsilon\|x\|_2\right)\leq 2 \exp \left(-c_0(\epsilon) m\right), \quad  \epsilon \in(0,1),
$$
where $\mathbb{E}(\|\Phi x\|^2) = \|x\|^2$ and $c_0(\epsilon)>0$ depends only on $\epsilon$. 
\end{definition}

\begin{theorem}[\cite{baraniuk2008simple}, Theorem 5.2] \label{trn:RIP}
Suppose that $m, n \in \mathbb{N}$ and $\delta \in(0,1)$ are given. If $\Phi \in \mathbb{R}^{m \times n}$ is a random matrix that is strongly concentrated around its expectation, then there exist $c_1, c_2>0$ depending only on $\delta$ such that the RIP holds for $\Phi$ with the prescribed $\delta$ and is of order $k$ obeying $k \leq c_1 m / \log (n / k)$ with probability at least $1-2 e^{-c_2 m}$.
\end{theorem}

Next, we demonstrate that a specific class of random variables called subgaussian random variables satisfy the requirements of Theorem \ref{trn:RIP}.
\begin{definition}[Subgaussian \cite{subgaussian}]
 A random variable $Z$ is said to be $b$-subgaussian, i.e., $Z\sim$ Sub($b^2$), if there exist some $b>0$ such that for every $t \in \mathbb{R}$,
$$
\mathbb{E}(e^{t Z}) \leq e^{b^2 t^2 / 2}.
$$
\end{definition}

\begin{theorem}[\cite{devore2009instance}, Lemma 6.1]\label{trn:subRIP} Suppose that $\Phi \in \mathbb{R}^{m \times n}$ is a random matrix with i.i.d. entries from a $b$-subgaussian distribution with variance $\frac{1}{m}$. Then for all $x \in \mathbb{R}^n$, we have $\mathbb{E}\left(\|\Phi x\|_2^2\right)=\|x\|_2^2$ and
$
\mathbb{P}\left(| \|\Phi x\|_2^2-\|x\|_2^2| \geq \epsilon\|x\|_2^2  \right)\leq 2 \exp \left(-\kappa \epsilon^2 m\right)
$ for some $\kappa>0$ and $\epsilon \in(0,1)$.

\end{theorem}
Now, we demonstrate that a uniform distribution satisfies the requirements of Theorem \ref{trn:RIP} by first showing that it is subgaussian.

\begin{lemma}\label{trn:uni}
Suppose random variable $Z$ is uniformly distributed over the interval $[-b, b]$ for some fixed $b>0$. Then $Z$ is $b$-subgaussian.
\end{lemma}
\begin{proof}
For a uniform distribution,
\begin{align*}
\mathbb{E} e^{t Z} & =\frac{1}{2 b} \int_{-b}^b e^{t x} d x=\frac{1}{2 b t}\left[e^{b t}-e^{-b t}\right]=\sum_{n=0}^{\infty} \frac{(b t)^{2 n}}{(2 n+1) !} \leq \sum_{n=0}^{\infty} \frac{(b t)^{2 n}}{n ! 2^n} = e^{b^2 t^2/2},
\end{align*}
where the inequality is from the fact that $(2 n+1) !\geq n ! 2^n$. 
\end{proof}

\subsection{Proof of Lemma \ref{col:uni_rip}}

\uniformRip*
\begin{proof}

We show a random sensing matrix $\textbf{H}_{m \times n}$, whose entries are drawn from an i.i.d. Uniform distribution in the range $\left(-\sqrt{3/m},\sqrt{3/m}\right)$, will be strongly concentrated around its expectation. According to Lemma \ref{trn:uni}, the Uniform distribution is a subgaussian distribution. From Theorem \ref{trn:subRIP}, matrices from subgaussian distributions with variance $1/m$ are strongly concentrated around their expectations. Finally, by Theorem \ref{trn:RIP},  matrix $\textbf{H}_{m \times n}$ which is strongly concentrated around its expectation, will satisfy the Restricted Isometry Property with at least probability $1-2 e^{-\kappa m}$ for some constant $\kappa>0$.
\end{proof}

\section{Appendix to Section~\ref{lapSec}}\label{sasd}

\subsection{Proof of Theorem~\ref{lapSP}}

\thmdelSP*
\begin{proof}

To prove Theorem \ref{lapSP}, we find the $l_1$-global sensitivity  for statistical parity gap $SP$. 

Let $\lambda=\sum_{\{j|a_j=1\}} h_i(x_j)$ and $\mu=\sum_{\{j|a_j=0\}} h_i(x_j)$. We assume without loss of generality that $N_0\leq N_1$, and there is at least $1$ individual in the disadvantaged group, i.e., $N_0 \geq 1$ in all possible datasets. A neighboring dataset is one such that the protected attribute differs for only one individual. Consider these two cases.

 \noindent  Case 1: An individual in the disadvantaged group, which the model predicted an unfavorable outcome, differs in the neighboring dataset in its protected attribute, i.e., $a_j=0$ with $h(x_j)=0$ becomes $a'_j=1$ with $h(x_j)=0$.

\begin{align*}
 & |SP(\mathcal{S},h(\cdot))-SP(\mathcal{S}',h(\cdot))| \stackrel{\text{(a)}}{=} \bigg | \frac{\lambda}{N_1} -\frac{\mu}{N_0}- \bigg (\frac{\lambda}{N_1+1}-\frac{\mu}{N_0-1}\bigg)\bigg| \\
 & = \frac{\lambda}{N_1(N_1+1)}+\frac{\mu}{N_0(N_0-1)}
 \stackrel{\text{(b)}}{\leq}  \frac{N_1}{N_1(N_1+1)}+\frac{N_0-1}{N_0(N_0-1)} \\
 & = \frac{1}{N_1+1}+\frac{1}{N_0} \stackrel{\text{(c)}}{\leq} \frac{1}{n-1}+\frac{1}{2}.
\end{align*}
where (a) holds since the size of the disadvantaged group decreased by one and that of the advantaged group increased. $\lambda, \mu$ was not incremented because $h(x_j)=0$. Next, (b) holds since $\lambda$ is upper bounded by $N_1$, while $\mu$ is upper bounded by $N_0-1$ since $h(x_j)=0$. Notice, this is the local sensitivity expression for the reference dataset. Hence, (c) is a maximization over all possible datasets to find the global sensitivity. The maximum occurs when $N_0=2$ and $N_1=n-2$, meaning, this dataset and the neighboring dataset with $N_0=1$ and $N_1=n-1$ have the largest sensitivity.

 \noindent  Case 2: An individual in the disadvantaged group, which the model predicted a favorable outcome, differs in the neighboring dataset in its protected attribute, i.e., $a_j=0$ and $h(x_j)=1$ becomes $a'_j=1$ and $h(x_j)=1$. Steps follow similar arguments with case 1.

\begin{align*}
 & |SP(\mathcal{S},h(\cdot))-SP(\mathcal{S}',h(\cdot))|  \stackrel{\text{(a)}}{=} \bigg | \frac{\lambda}{N_1} -\frac{\mu}{N_0}- \bigg(\frac{\lambda+1}{N_1+1}-\frac{\mu-1}{N_0-1}\bigg)\bigg| \\
 & =  \frac{N_1-\lambda}{N_1(N_1+1)}+\frac{N_0-\mu}{N_0(N_0-1)}
  \stackrel{\text{(b)}}{\leq} \frac{N_1}{N_1(N_1+1)}+\frac{N_0-1}{N_0(N_0-1)}\\
  & = \frac{1}{N_1+1}+\frac{1}{N_0}
   \stackrel{\text{(c)}}{\leq} \frac{1}{n-1}+\frac{1}{2}.
\end{align*}

 The maximum $l_1$ difference between any two neighboring datasets is upper bounded by $1/(n-1)+1/2$. Therefore with $m$ queries, we have the global sensitivity $\Delta_{SP}=m/(n-1)+m/2$.
\end{proof}

\subsection{Proof of Theorem~\ref{ASPD}}

 \thmabsSP*
%  \begin{theorem}\label{ASPD}
% For absolute statistical parity gap $|SP|$, the Laplace mechanism that adds noise from  $\rm{Lap}({\Delta_{SP}}/{\epsilon})$ is $\epsilon$-differential private where $\Delta_{|SP|}= \frac{m}{2}$.
% \end{theorem}
\begin{proof}
To prove Theorem \ref{ASPD}, we find the $l_1$-global sensitivity for absolute statistical parity gap $|SP|$.

Let $\lambda=\sum_{\{j|a_j=1\}} h_i(x_j)$ and $\mu=\sum_{\{j|a_j=0\}} h_i(x_j)$. We assume without loss of generality that $N_0\leq N_1$, and there is at least 1 individual in the disadvantaged group, i.e., $N_0 \geq 1$, in all possible datasets. A neighboring dataset is one such that the protected attribute differ by one individual. Consider these two cases.

\noindent  Case 1: An individual in the disadvantaged group, which the model predicted an unfavorable outcome, differs in the neighboring dataset in its protected attribute, i.e., $a_j=0$ with $h(x_j)=0$ becomes $a'_j=1$ with $h(x_j)=0$.

\begin{align*}
 &\left||SP|(\mathcal{S},h(\cdot))-|SP||(\mathcal{S}',h(\cdot))|\right|
 \stackrel{\text{(a)}}{=} \Bigg | \bigg|\frac{\lambda}{N_1} -\frac{\mu}{N_0}\bigg|- \bigg |\frac{\lambda}{N_1+1}-\frac{\mu}{N_0-1}\bigg|\Bigg|\\
& \stackrel{\text{(b)}}{\leq} \Bigg | \bigg|\frac{0}{N_1} -\frac{N_0-1}{N_0}\bigg|- \bigg |\frac{0}{N_1+1}-\frac{N_0-1}{N_0-1}\bigg|\Bigg| = \frac{1}{N_0} \stackrel{\text{(c)}}{\leq} \frac{1}{2}.
\end{align*}
Here, (a) holds since the protected attribute of the individual was flipped from $a_j=0$ to $a'_j=1$, $N_0$ decreased by one, and $N_1$ increased by one. $\lambda, \mu$ were not incremented since $h(x_j)=0$. To obtain the upper bound in (b), $\lambda$ is lower bounded by $0$, while $\mu$ is upper bounded by $N_0-1$ since $h(x_j)=0$. Notice, this is the local sensitivity expression for the reference dataset. Hence, (c) is a maximization over all possible datasets to find the global sensitivity. The maximum occurs when $N_0=2$.

\noindent  Case 2: An individual in the disadvantaged group, which the model predicted a favorable outcome, differs in the neighboring dataset in its protected attribute, i.e., $a_j=0$ with $h(x_j)=1$ becomes $a'_j=1$ with $h(x_j)=1$. Steps follow similar arguments with case 1.

\begin{align*}
& \left||SP|(\mathcal{S},h(\cdot))-|SP|(\mathcal{S}',h(\cdot))|\right|
 \stackrel{\text{(a)}}{=} \Bigg | \bigg|\frac{\lambda}{N_1} -\frac{\mu}{N_0}\bigg|- \bigg |\frac{\lambda+1}{N_1+1}-\frac{\mu-1}{N_0-1}\bigg|\Bigg|\\
& \stackrel{\text{(b)}}{\leq} \Bigg | \bigg|\frac{N_1}{N_1} -\frac{1}{N_0}\bigg|- \bigg |\frac{N_1+1}{N_1+1}-\frac{1-1}{N_0-1}\bigg|\Bigg| = \frac{1}{N_0} \stackrel{\text{(c)}}{\leq} \frac{1}{2}.
\end{align*}

The maximum $l_1$ difference between any two neighboring datasets is upper bounded by $1/2$. Therefore with $m$ queries, we have the global sensitivity $\Delta_{|SP|}=m/2$.
 
\end{proof} 

\section{Appendix to Section~\ref{SmoSen}}
\label{apx:ss}

\subsection{Background: Relevant Results in Smooth Sensitivity}
We first define an admissible noise distribution. 
% \begin{definition}[\cite{??} Def. 3.1.] The sensitivity of $f$ at distance $k$ is
% $$
% \max _{y \in D^n: d(x, y) \leq k} L S_f(y) .
% $$
% Now smooth sensitivity can be expressed in terms of $A^{(k)}$ :
% $$
% \begin{aligned}
% S_{f, \epsilon}^*(x) & =\max _{k=0,1, \ldots, n} e^{-k \epsilon}\left(\max _{y: d(x, y)=k} L S_f(y)\right) \\
% & =\max _{k=0,1, \ldots, n} e^{-k \epsilon} A^{(k)}(x) .
% \end{aligned}
% $$
% \end{definition}

\begin{definition}[\cite{10.1145/1250790.1250803}, Def. 2.5 (Admissible Noise Distribution)] A probability distribution on $\mathbb{R}^m$, given by a density function $p$, is $(\alpha, \beta)$-admissible (with respect to $\ell_1$) if, for $\alpha=\alpha(\epsilon, \delta), \beta=\beta(\epsilon, \delta)$, the following two conditions hold for all $\Delta \in \mathbb{R}^m$ and $\lambda \in \mathbb{R}$ satisfying $\|\Delta\|_1 \leq \alpha$ and $|\lambda| \leq \beta$, and for all measurable subsets $\tau \subseteq \mathbb{R}^m$:
$$
\begin{array}{ll}
\text { Sliding Property: } & \operatorname{Pr}_{Z \sim p}[Z \in \tau] \leq e^{\frac{\epsilon}{2}} \cdot \operatorname{Pr}_{Z \sim p}[Z \in \tau+\Delta]+\frac{\delta}{2} . \\
\text { Dilation Property: } & \operatorname{Pr}_{Z \sim p}[Z \in \tau] \leq e^{\frac{\epsilon}{2}} \cdot \operatorname{Pr}_{Z \sim p}\left[Z \in e^\lambda \cdot \tau \right]+\frac{\delta}{2}.
\end{array}
$$    
\end{definition}
Next, we show how an admissible noise distribution helps achieve differential privacy. 
\begin{lemma}[\cite{10.1145/1250790.1250803}, Lemma 2.6.] \label{lemma26}Let $h$ be an $(\alpha, \beta)$-admissible noise probability density function, and let $Z$ be a fresh random variable sampled according to $h$. For a function $f: \mathcal{D} \rightarrow \mathbb{R}^m$, let $S: D^n \rightarrow \mathbb{R}$ be a $\beta$-smooth sensitivity of $f$. Then algorithm $\mathcal{M}(x)=f(x)+\frac{S(x)}{\alpha} \cdot Z$ is $(\epsilon, \delta)$-differentially private.
\end{lemma}

\subsection{Proofs of Theorems \ref{caucy} and \ref{thm:lastlap}}

The following two results will help prove Theorem \ref{caucy} and \ref{thm:lastlap} in the paper.
\begin{lemma}[\cite{10.1145/1250790.1250803}, Lemma 2.7.]\label{lemma27} For any $\gamma>1$, the distribution with density $p(z) \propto \frac{1}{1+|z|^\gamma}$ is $\left(\frac{\epsilon}{2(\gamma+1)}, \frac{\epsilon}{2(\gamma+1)}\right)$-admissible (with $\delta=0)$. Moreover, the m-dimensional product of independent copies of $p$ is $\left(\frac{\epsilon}{2(\gamma+1)}, \frac{\epsilon}{2 m(\gamma+1)}\right)$-admissible.
\end{lemma}
For Theorem \ref{caucy}, observe that the standard Cauchy distribution $p(z) \propto \frac{1}{1+|z|^2}$ is $(\alpha,\beta)$-admissible with $\gamma=2$ from Lemma~\ref{lemma27}. Hence, the standard Cauchy distribution can be used to achieve $\epsilon$-differential privacy from Lemma \ref{lemma26}.

\begin{lemma}[\cite{10.1145/1250790.1250803}, Lemma 2.9.]\label{lemma29} For $\epsilon, \delta \in(0,1)$, the m-dimensional Laplace distribution, $p(z)=\frac{1}{2^m} \cdot e^{-\|z\|_1}$, is $(\alpha, \beta)$-admissible with $\alpha=\frac{\epsilon}{2}$, and $\beta=\frac{\epsilon}{2 \rho_{\delta / 2}\left(\|Z\|_1\right)}$, where $Z \sim p$. In particular, it suffices to use $\alpha=\frac{\epsilon}{2}$ and $\beta=\frac{\epsilon}{4(m+\ln (2 / \delta))}$. For $m=1$, it suffices to use $\beta=\frac{\epsilon}{2 \ln (2 / \delta)}$. The expression $\rho_{\delta/2}(\|Z\|_1)$ represents the $(1-\frac{\delta}{2})$ quantile of random variable $\|Z\|_1$.
\end{lemma}

For Theorem \ref{thm:lastlap}, we use Lemma~\ref{lemma29} to show that the Laplace distribution is also an admissible noise distribution and hence can be used to achieve $(\epsilon,\delta)$-differential privacy from Lemma~\ref{lemma26}.

% \rd{Definition 2.8. Given a real-valued random variable $Y$ and a number $\delta \in(0,1)$, let $\rho_\delta(Y)$ be the $(1-\delta)$ quantile of $Y$, that is, the least solution to $\operatorname{Pr}\left(Y \leq \rho_\delta\right) \geq 1-\delta$.}

\subsection{Additional Lemmas aiding proof of Theorem \ref{caucy}}\label{derivative} 
\begin{lemma}\label{lem:kdist}
Let dataset $\mathcal{S} \in \mathcal{D}$, containing two groups with sizes $N_0(\mathcal{S})$ and $N_1(\mathcal{S})$, such that $N_0(\mathcal{S}) \leq N_1(\mathcal{S})$. The maximum local sensitivity at distance $k$ can be expressed as follows:
   $$ \max_{\bar{S}:d(S,\bar{S})=k} \Delta^{local}_{SP}(\bar{S})
 =  \frac{m}{N_1(\mathcal{S})+k+1}{+}\frac{m}{N_0(\mathcal{S})-k}$$.
\end{lemma}
\begin{proof}

First, we will show that the local sensitivity  $\Delta^{local}_{SP}$ is 

   $$\Delta^{local}_{SP}(\mathcal{S})=\max_{S':d(S,S')=1} |SP(\mathcal{S})-SP(\mathcal{S}')|=\frac{m}{N_1(\mathcal{S})+1}+\frac{m}{N_0(\mathcal{S})}$$.

Let $\lambda=\sum_{\{j|a_j=1\}} h_i(x_j)$ and $\mu=\sum_{\{j|a_j=0\}} h_i(x_j)$. A neighboring dataset $\mathcal{S'}$ is one such that the protected attribute differs for only one individual, i.e., $d(S,S')=1$. Consider these two cases.

 \noindent  Case 1: An individual in the disadvantaged group, which the model predicted an unfavorable outcome, differs in the neighboring dataset in its protected attribute, i.e., $a_j=0$ with $h(x_j)=0$ becomes $a'_j=1$ with $h(x_j)=0$.

\begin{align*}
 & |SP(\mathcal{S})-SP(\mathcal{S}')|\stackrel{\text{(a)}}{=} \bigg | \frac{\lambda}{N_1(\mathcal{S})} -\frac{\mu}{N_0(\mathcal{S})}- \bigg (\frac{\lambda}{N_1(\mathcal{S})+1}-\frac{\mu}{N_0(\mathcal{S})-1}\bigg)\bigg| \\
 & = \frac{\lambda}{N_1(\mathcal{S})(N_1(\mathcal{S})+1)}+\frac{\mu}{N_0(\mathcal{S})(N_0(\mathcal{S})-1)}
 \stackrel{\text{(b)}}{\leq}  \frac{N_1}{N_1(\mathcal{S})(N_1(\mathcal{S})+1)}+\frac{N_0(\mathcal{S})-1}{N_0(\mathcal{S})(N_0(\mathcal{S})-1)} \\
 & = \frac{1}{N_1(\mathcal{S})+1}+\frac{1}{N_0(\mathcal{S})}.
\end{align*}
where (a) holds since the size of the disadvantaged group decreased by one and that of the advantaged group increased. $\lambda, \mu$ was not incremented because $h(x_j)=0$. Next, (b) holds since $\lambda$ is upper bounded by $N_1$, while $\mu$ is upper bounded by $N_0-1$ since $h(x_j)=0$.

 \noindent  Case 2: An individual in the disadvantaged group, which the model predicted a favorable outcome, differs in the neighboring dataset in its protected attribute, i.e., $a_j=0$ and $h(x_j)=1$ becomes $a'_j=1$ and $h(x_j)=1$. Steps follow similar arguments with case 1.

\begin{align*}
 & |SP(\mathcal{S})-SP(\mathcal{S}')|  \stackrel{\text{(a)}}{=} \bigg | \frac{\lambda}{N_1(\mathcal{S})} -\frac{\mu}{N_0(\mathcal{S})}- \bigg(\frac{\lambda+1}{N_1(\mathcal{S})+1}-\frac{\mu-1}{N_0(\mathcal{S})-1}\bigg)\bigg| \\
 & =  \frac{N_1(\mathcal{S})-\lambda}{N_1(\mathcal{S})(N_1(\mathcal{S})+1)}+\frac{N_0(\mathcal{S})-\mu}{N_0(\mathcal{S})(N_0(\mathcal{S})-1)}
  \stackrel{\text{(b)}}{\leq} \frac{N_1(\mathcal{S})}{N_1(\mathcal{S})(N_1(\mathcal{S})+1)}+\frac{N_0(\mathcal{S})-1}{N_0(\mathcal{S})(N_0(\mathcal{S})-1)}\\
  & = \frac{1}{N_1(\mathcal{S})+1}+\frac{1}{N_0(\mathcal{S})}.
\end{align*}

 For one query, the maximum $SP$ change between $\mathcal{S}$ and any neighboring dataset $\mathcal{S'}$ is upper bounded by $1/(N_1+1)+1/N_0$. Therefore with $m$ queries, we have the local sensitivity $\Delta^{local}_{SP(\mathcal{S})}=\frac{m}{N_1(\mathcal{S})+1}+\frac{m}{N_0(\mathcal{S})}$.

 Now, we find the maximum local sensitivity over datasets that differ by $k$ entries (maximum local sensitivity at distance $k$). 
\begin{align*}
\max_{\bar{S}:d(S,\bar{S})=k} \Delta LS_{SP}(\bar{S})
&= \max_{\bar{S}:d(S,\bar{S})=k} \max_{\bar{S}':d(\bar{S},\bar{S}')=1} |SP(\bar{\mathcal{S}})-SP(\bar{\mathcal{S}}')|\\
 & = \max_{\bar{S}:d(S,\bar{S})=k}\frac{m}{N_1(\bar{\mathcal{S}})+1}+\frac{m}{N_0(\bar{\mathcal{S}})}  \\ & \stackrel{\text{(a)}}{=} \frac{m}{N_1(\mathcal{S})+k+1}{+}\frac{m}{N_0(\mathcal{S})-k}.
\end{align*}

The last equality (a) holds since the maximum when considering datasets at distance $k$ is attained when the protected attribute of $k$ individuals in the disadvantaged group is moved to the advantaged group.
\end{proof}
\begin{remark}
Similarly, when $N_0(\mathcal{S}) \geq N_1(\mathcal{S})$, the maximum local sensitivity at distance $k$ can be expressed as follows:   $$\max_{\bar{S}:d(S,\bar{S})=k} \Delta^{local}_{SP}(\bar{S})
 =  \frac{m}{N_0(\mathcal{S})+k+1}{+}\frac{m}{N_1(\mathcal{S})-k}.$$
\end{remark}
\begin{lemma}\label{lem:maxb} Given $m,\beta,N_0,N_1>0$, with $N_0 \leq N_1$, the maximum of the expression occurs at the boundaries $k=\{0,N_0-2\}$, i.e.,

 $$ \max_{k=0,1,...,N_0-2} e^{-k\beta} \bigg( \frac{m}{N_1+k+1}+\frac{m}{N_0-k}\bigg)= \max \bigg( \frac{m}{N_1+1}+\frac{m}{N_0}\; ,\; e^{-(N_0-2)\beta}\left(\frac{m}{n-1}+\frac{m}{2}\right) \bigg). $$
 
 \end{lemma}
 \begin{proof}
 To prove the maximum of the function is always at its boundaries, it suffices to show the function is convex on that interval. We assume the function is continuous in $k$ and show the double derivative is greater than zero.
 % $$ \max_{k=0,1,...,N_0-2} e^{-k\beta} \bigg( \frac{m}{N_1+k+1}+\frac{m}{N_0-k}\bigg) $$
 \begin{align*}
\frac{d^2}{d^2k} e^{-k\beta} \bigg( \frac{m}{N_1+k+1}+\frac{m}{N_0-k}\bigg) = \beta^2 e^{-k\beta}\left( \frac{m}{N_1+k+1}+\frac{m}{N_0-k}\right)\\
 + e^{-k\beta}\left( \frac{2m}{(N_1+k+1)^3}+\frac{2m}{(N_0-k)^3}\right)\\
  +2\beta e^{-k\beta}\left( \frac{m}{(N_1+k+1)^2}-\frac{m}{(N_0-k)^2}\right) \stackrel{\text{(a)}}{>} 0
 \end{align*}
To prove the derivative is greater than zero, notice that,
\begin{align}
    \frac{\beta^2}{(N_0-k)}+\frac{2}{(N_0-k)^3} -  \frac{2 \beta}{(N_0-k)^2} > 0
\end{align}
This holds because the quadratic equation has a discriminant $\frac{4 \beta^2}{(N_0-k)^4} - \frac{8 \beta^2}{(N_0-k)^4} < 0$. This implies that the parabola is either entirely above the $x$-axis.

Hence,  for $m,\beta,N_0,N_1>0$, and $k=0,1,...,N_0-2$, inequality (a) holds proving strict convexity. The maximum of a strictly convex function occurs at one of the extreme points $k=\{0,N_0-2\}$.
 % $$ \max \bigg( \frac{m}{N_1+1}+\frac{m}{N_0}\; ,\; e^{-(N_0-2)\beta}\left(\frac{m}{n-1}+\frac{m}{2}\right) \bigg)$$
\end{proof}

\subsection{$\beta$-Smooth Sensitivity of Absolute Statistical Parity}\label{AAS}

\begin{lemma}
The $\beta$-smooth sensitivity of the absolute statistical parity gap $|SP|$ query is given as:
$$\Delta^{smooth}_{|SP|,\beta}(\mathcal{S})=\max \bigg(\frac{m}{N_0} , \frac{me^{-(N_0-2)\beta}}{2} \bigg),$$ 
where $N_0$ and $N_1$ are sizes of disadvantaged and advantaged groups in dataset $\mathcal{S}$, such that $N_0 \leq N_1$.
\end{lemma}
\begin{proof}

 \begin{align*}
 \Delta^{smooth}_{|SP|,\beta}(\mathcal{S}) & \stackrel{\text{(a)}}{=} \max_{\mathcal{\bar{S}} \in \mathcal{D} } \bigg(\Delta LS_{|SP|}(\mathcal{\bar{S}}) \cdot e^{-\beta d(\mathcal{S},\mathcal{\bar{S}})}\bigg) \\
 & \stackrel{\text{(b)}}{=} \max_{k=0,1,...,N_0-2} e^{-k\beta} \bigg( \max_{\bar{S}:d(S,\bar{S})=k} \Delta LS_{SP}(\bar{S})\bigg)\\
  & \stackrel{\text{(c)}}{=} \max_{k=0,1,...,N_0-2} e^{-k\beta} \bigg(\frac{m}{N_0-k}\bigg) \\
  & \stackrel{\text{(d)}}{=} \max \bigg(\frac{m}{N_0} , \frac{me^{-(N_0-2)\beta}}{2} \bigg)
 \end{align*}
 Here, (a) follows from the definition of smooth sensitivity in \eqref{def:SS}. Next, (b) is the expression of the smooth sensitivity when looking at datasets at distance $k$ (for details see \cite[Def 3.1]{10.1145/1250790.1250803}). To obtain (c) we find the maximum local sensitivity over datasets $k$ distance away (similar to Lemma \ref{lem:kdist} in Appendix~\ref{derivative}). Finally, (d) holds since the function is convex, and hence its maximum occurs at the boundaries $k \in \{0, N_0-2\}$ (similar to Lemma \ref{lem:maxb} in Appendix~\ref{derivative}).
 \end{proof}
\section{Additional Experimental Results}\label{expref}
In this section, we provide more experimental results to validate our theoretical findings. \\

\textbf{Performance metrics} 
\begin{enumerate}
    \item  Average error in answering the Statistical Parity query (Avg. SP Err) 
    
    \item  Leakage: Accuracy of correctly recovering the protected attribute balanced across both races (see Definition \ref{leak}).
% \begin{definition}[Leakage(\%)]
% \label{leak} Let ${N_A}$ be number of individuals in the advantaged group whose protected attribute was correctly predicted and ${N_B}$ be number of individuals in the disadvantaged group whose protected attribute was correctly predicted.  The protected attribute leakage is defined as;
% $$\text{Leakage}= \frac{1}{2} \left( \frac{{N_A}}{{N_1 }}+\frac{{N_B}}{{N_0}}\right) \times 100$$
%  The leakage is the balanced accuracy of recovery. This is used to deal with imbalanced data, i.e. one target class appears a lot more than the other.
% \end{definition}
\item  Number of correctly recovered protected attributes $N_A+N_B$.
\end{enumerate}

\noindent \textbf{Experiments with German Credit dataset: }
The German credit dataset has 20 attributes for 1000 loan applicants. The classification target is whether an individual has a Good or Bad credit risk~\cite{Dua:2019a}. The protected attribute gender is excluded from training and only used for statistical parity evaluation. In the case the whole dataset is used for fairness evaluation, we have $N_1=690$ and $N_0=310$. We train a base model $h_0(\cdot)$ for classification. Our model achieved a reasonable accuracy of $73.70\%$.  For the other $m$ models, we add a small noise sampled from Uniform$(-0.1,0.1)$ distribution to each output of the base model. They had a mean accuracy of $73.44\%$ and standard deviation of $0.3942$.

\begin{table}[h]

\centering
\caption{Avg. SP err and No. of recovered protected attribute for different values of $\epsilon$ and $m$.}
\begin{tabular}{ccccc}
\hline
$\epsilon$ & \multicolumn{2}{l}{Avg. SP err ($\times 10^{-3}$)}        & \multicolumn{2}{l}{No of recovered protected attribute} \\ \cline{2-5} 
                            & $m = 800$             & $m = 1000$            & \;\;$m = 800$                  & $m = 1000$                 \\ \hline
10                          & 133.7   & 112.9   & 558                        &\;\; 492                        \\
100                         & 6.1    & 6.3  & 557                        &\;\; 500                        \\
1000                        & 0.5 & 0.7 & 657                        &\;\; 516                        \\ \hline
\end{tabular}
\label{table1}
\end{table}

\begin{remark} We note that it is unreasonable to expect that no protected attribute would be recovered correctly. This is because even if one tries to guess the protected attributes based on a random coin flip, they will accurately recover at least some protected attributes. The goal of privacy is to rather ensure that the protected attribute of no \emph{targeted} individual is leaked with certainty. We are therefore interested in privatizing the protected attributes so that the performance is not better than random guessing, i.e., getting $50\%$ recovery accuracy.
\end{remark}

 To demonstrate the effectiveness Attribute-Conceal in protecting the protected attribute,  we use the mechanism in Theorem~\ref{caucy} to achieve $\epsilon$-DP (see Algorithm~\ref{algo:conceal}). Figure~\ref{fig2} (a) shows the number of protected attributes recovered using CS when noise is added to the statistical parity queries. Unlike the noiseless case, recovery is no longer possible even with more models. In Figure~\ref{fig2} (b) and (d), we plot the recovered  $\bar{s}$ vector and the answered statistical parity gap queries with $\epsilon=1000$. We experiment with different values of $\epsilon$ to study the trade-off between average statistical parity query error and recovery performance. Our results are summarized in Table \ref{table1}. By increasing $\epsilon$, there is a decrease in the average SP query error without a significant increase in number of protected attributes recovered. We show this for $m=800$ and $m=1000$ models.  The model developers generally do not have control over the size of the advantaged or disadvantaged group. This dataset has a $70:30$ male to female ratio. Despite the fact that $N_0$ is not substantially less than $N_1$, and our sensing matrix is uniformly distributed, full recovery of protected attributes needed fewer models. 
 
\begin{figure*}[h]
\begin{minipage}[b]{0.22\textwidth}
\includegraphics[scale=0.2]{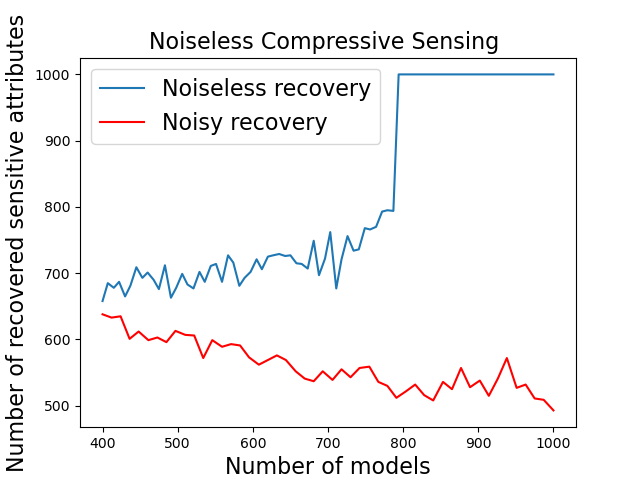}
\end{minipage}
\begin{minipage}[b]{0.29\textwidth}
\includegraphics[scale=0.23]{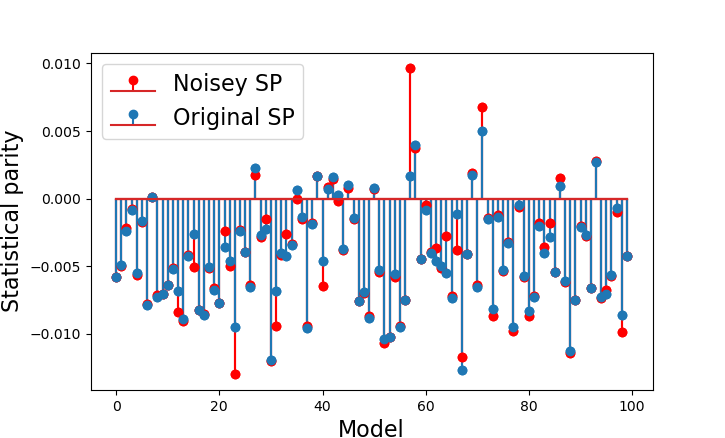}
\end{minipage}
\begin{minipage}[b]{0.23\textwidth}
\includegraphics[scale=0.23]{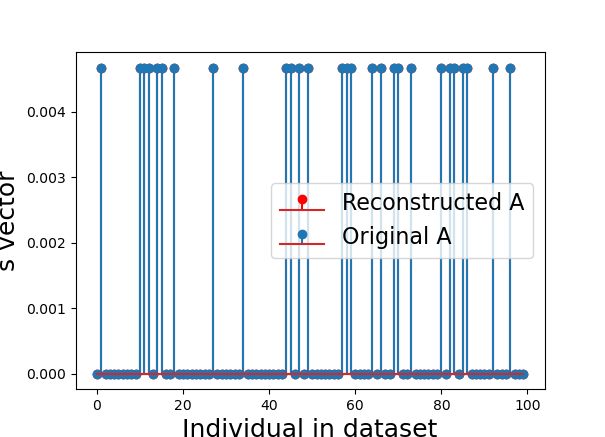}
\end{minipage}
\begin{minipage}[b]{0.17\textwidth}
\includegraphics[scale=0.21]{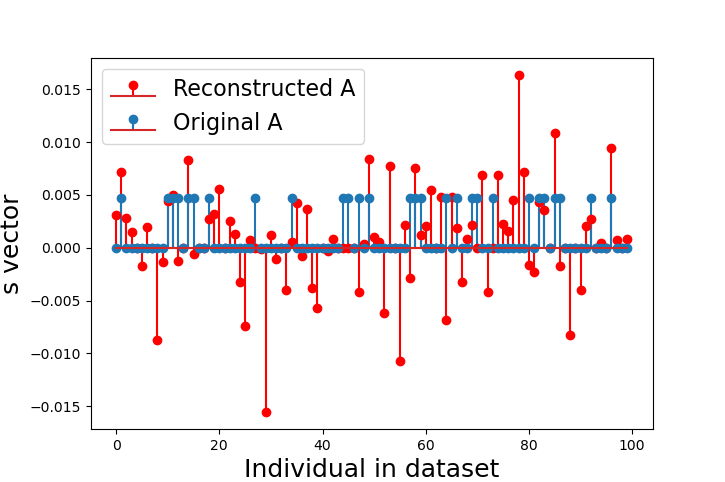}
\end{minipage}
\caption{(a) Number of recovered protected attributes as a function of number of models with and without Attribute-Conceal ($\epsilon=1000$). Perfect recovery can be achieved by using $800$ or more models for noiseless case. (b) The original and answered statistical parity (SP) queries for $m=800$ models  under $\epsilon=1000$ with an average error of $4.78 \times 10^{-4}$.
(c) Recovered $\bar{s}$ vector used to infer the protected attribute $A$ for $m=800$. There is an overlap between original and recovered $\bar{s}$ vector. (d) The original and reconstructed $\bar{s}$ vector with Attribute-Conceal, $\epsilon=1000$. (First $100$ individuals and models are shown for clarity).}
\label{fig2}
\end{figure*}

\end{document}